\documentclass{elsarticle}
\usepackage[utf8]{inputenc} 
\usepackage[T1]{fontenc}    
\usepackage{hyperref}       
\usepackage{url}            
\usepackage{booktabs}       
\usepackage{amsfonts}       
\usepackage{nicefrac}       
\usepackage{microtype}      
\usepackage{lipsum}		
\usepackage{graphicx}
\usepackage{natbib}
\usepackage{doi}
\usepackage{orcidlink}
\usepackage{subfigure}
\usepackage{amsmath,amsfonts,amssymb,amsthm}
\usepackage{mdframed}
\usepackage{makecell}
\usepackage{setspace}
\usepackage{enumerate}
\usepackage{array}
\usepackage{multirow}
\usepackage{comment}
\usepackage{float}
\usepackage{optidef}
\usepackage{lscape}
\usepackage{wrapfig}
\usepackage{hyperref}
\usepackage{cleveref}
\usepackage{mdframed}
\usepackage{enumitem, array}
\usepackage{rotating}
\usepackage{hyperref}
\usepackage[procnumbered,boxruled,linesnumbered,lined,boxed,commentsnumbered]{algorithm2e}
\usepackage{afterpage}
\newtheorem{theorem}{Theorem}[section]

\newtheorem{lemma}[theorem]{Lemma}

\DeclareMathOperator{\Tr}{Tr}

\definecolor{awesome}{rgb}{1.0, 0.13, 0.32}

\SetCommentSty{mycommfont}

\title{Generative Models for Anomaly Detection and Design-Space Dimensionality Reduction in Shape Optimization}

\author{{Danny D'Agostino\hspace{0.01mm}
\orcidlink{https://orcid.org/0000-0003-0028-4720}}\\
Department of Industrial Systems Engineering and Management \\
	National University of Singapore\\
	Singapore \\
	\texttt{dannydag@nus.edu.sg} \\
}

\begin{document}
%
\begin{abstract}
Our work presents a novel approach to shape optimization, with the twofold objective to improve the efficiency of global optimization algorithms while promoting the generation of high-quality designs during the optimization process free of geometrical anomalies. 
This is accomplished by reducing the number of the original design variables defining a new reduced subspace where the geometrical variance is maximized and modeling the underlying generative process of the data via probabilistic linear latent variable models such as factor analysis and probabilistic principal component analysis. We show that the data follows approximately a Gaussian distribution when the shape modification method is linear and the design variables are sampled uniformly at random, due to the direct application of the central limit theorem.
The degree of anomalousness is measured in terms of Mahalanobis distance, and the paper demonstrates that abnormal designs tend to exhibit a high value of this metric. 
This enables the definition of a new optimization model where anomalous geometries are penalized and consequently avoided during the optimization loop.
The procedure is demonstrated for hull shape optimization of the DTMB 5415 model, extensively used as an international benchmark for shape optimization problems. The global optimization routine is carried out using Bayesian optimization and the DIRECT algorithm.
From the numerical results, the new framework improves the convergence of global optimization algorithms, while only designs with high-quality geometrical features are generated through the optimization routine thereby avoiding the wastage of precious computationally expensive simulations.
\end{abstract}
\maketitle
%
\section{Introduction}
In the last decades, we obtained incredible progress in computational power and mathematical tools, but engineering design processes have also reached a high level of complexity that makes the design optimization process still challenging. This is especially prominent in simulation-based design optimization (SBDO) when time-consuming high-fidelity simulators are considered. 
Furthermore, one of the most complex challenges is how to deal with high-dimensional large design spaces,  when computationally-expensive black-box functions are used for the performance analysis and a global optimum is sought after. 
Even if efficient global optimization algorithms have been proposed \cite{jones1993-JOTA, kennedy1995particle, mockus1978application} and applied with success to SBDO, finding a potentially global optimal solution within reasonable computational time/cost remains a critical issue and a technological challenge.

For those reasons, proposing methodologies capable of weakening the computational complexity of the overall optimization process is still crucial, and recently, has been exploring the possibility of reducing the number of design variables of the optimization problem, through dimensionality reduction machine learning models such as principal component analysis (PCA) \cite{hotelling1933analysis,pearson1901liii} to alleviate the well-known "curse of dimensionality" \cite{bellman1966dynamic}. 

Beyond dimensionality reduction, generative models have gained prominence in the domain of engineering design optimization. Generative models learn the joint probability distribution function (PDF) of the observed data $p(\mathbf{x})$ \cite{goodfellow2016deep, bishop2006-PRM, gelman2013bayesian}. There are several different kinds of tasks for which we can use generative models such as generating synthetic data, density estimation, and data imputation \cite{pml2Book}. In the context of our work, both dimensionality reduction and density estimation play a central role where the latter refers to evaluating the probability of an observed data vector, (i.e. computing $p(\mathbf{x})$). Some of these models can directly describe the probability distribution in a closed form such as probabilistic PCA (PPCA) \cite{roweis1997algorithms, tipping1999probabilistic} and factor analysis (FA) \cite{bartholomew1995spearman, spearman1961general} so that $p(\mathbf{x})$ can be evaluated explicitly. Other models such as generative adversarial networks (GANs) \cite{goodfellow2020generative} and variational autoencoders (VAEs)\cite{kingma2013auto} do not allow direct computation of $p(\mathbf{x})$ but allow actions that implicitly require knowledge of it (e.g. generating synthetic samples). 

Once the density estimation task is completed, is possible to use the trained generative model for \textit{anomaly detection}.
Anomaly detection is the process of identifying data points or patterns that do not conform to the expected or normal behavior within a dataset \cite{chandola2009anomaly}. These anomalies can manifest as unusual values, patterns, or events that differ significantly from the majority of the data points. 
The core principle of density estimation for anomaly detection is simple: anomalies are instances residing in regions of low data density, as they are less likely to conform to the patterns established by the learned data distribution. 
This process plays a crucial role in various domains, including finance \cite{ahmed2016survey}, cybersecurity \cite{bhuyan2013network}, industrial quality control \cite{QUATRINI2020117}, and healthcare \cite{ukil2016iot}, where the detection of unusual events can have critical implications. 
In engineering design optimization, the ability to identify such anomalies is of primary importance as it helps engineers avoid exploring impractical or suboptimal design solutions, especially in the presence of expensive computer simulations.

\section{Related Work}
In the context of shape optimization, research has traditionally focused on the shape and topology parameterization, as critical factors to achieve the desired level of design variability \cite{rozvany1992-SO,bletzinger1997-EO,samareh2001-AIAA}.
The choice of the shape parameterization technique has a large impact on the practical implementation and the success of the optimization process. Shape deformation methods have been an area of continuous and extensive research within the fields of computer graphics and geometry modeling. Consequently, many techniques have been proposed in recent years \cite{sieger2015-NCGGASC}. Several techniques have been developed and applied \cite{samareh2001-AIAA}, such as: basis vector methods \cite{pickett1973-AIAA}, domain element and discrete approaches \cite{leiva1998-AIAA}, partial differential equation \cite{bloor1995-JA}, CAD-based \cite{lacourse1995-HSM}, analytical \cite{hicks1978-AIAA}, polynomials \cite{haftka1986-CMAME} and the popular free-form deformation (FFD) \cite{sederberg1986-ACM}. 

For the SBDO to alleviate the curse of dimensionality and be successful, the parameterization method must efficiently describe the design variability with as few variables as possible.
A linear dimensionality reduction model based on the PCA (also known as proper orthogonal decomposition (POD) \cite{berkooz1993proper} or Karhunen Loeve expansion (KLE) \cite{karhunen1950struktur}) has been applied for learning reduced-dimensionality representations of the original shape parametrization. 
A framework based on the KLE and POD has been formulated in \cite{ghoman2012pod}, and similarly in \cite{diez2015-CMAME} for the assessment of the shape modification variability and the definition of a reduced-dimensionality global model of the shape modification vector, for arbitrary modification methods. 

KLE/PCA methods have been successfully applied in the case of the shape optimization of a high-speed Delft catamaran optimized with deterministic particle swarm optimization in \cite{chen2015-EO}, for single and multidisciplinary design optimization of a destroyer ship and a 3D hydrofoil in \cite{diez2016-AIAA}, and for the reparametrization of a destroyer ship using different shape modifications methods in \cite{dagostino2020design}. In \cite{marino2018shape} they coupled the POD and the dynamic mode decomposition \cite{schmid2010dynamic} to speed up the hull form optimization of a cruise ship. In  \cite{diez2015-FAST_SO,diez2018-SMO} they used stochastic optimization in the case of mono-hulls and catamarans in calm water and waves, respectively. Similarly, in \cite{poole2017-CF} have applied POD to airfoil shape optimization via singular value decomposition (SVD) of an airfoil geometric-data library.
Optimization methods enhanced by POD are proposed for aerodynamic shape optimization also in \cite{zhang2022two} and \cite{cinquegrana2018investigation}.

Those methods improve shape optimization efficiency by reparametrization and dimensionality reduction, providing the assessment of the design space and the shape parametrization before optimization and/or performance analysis are carried out. The assessment is based on the geometric variability associated with the design space, making the method computationally very efficient and attractive, as no simulations are required. 
The methodology has been extended allowing effective dimensionality reduction in presence of high nonlinearities due to complex geometrical scaling procedures in \cite{dagostino2017-MOD,dagostino2018-AIAA-MAO,dagostino2018-AIAASciTech} and for a physics-informed formulation in \cite{dagostino2018-LOD,serani2018-SNH,serani2020assessing}. 

Recently generative models have been used in the context of engineering design.
Most of the related work in this domain involves the usage of generative deep learning models as in \cite{chen2021padgan} where they developed a Performance-Augmented Diverse Generative Adversarial Network (PaDGAN) for automatic design synthesis and design space exploration using three synthetic examples and one real-world airfoil design example. In \cite{KHAN2023116051}, they introduced ShipHullGAN trained on a dataset of physically validated designs from a wide range of existing ship types, including container ships, tankers, bulk carriers, tugboats, and crew supply vessels.
In \cite{chen2020airfoil} they improved the airfoil optimization process by developing a Bézier-GAN, parametrizing aerodynamic designs by learning from shape variations in an existing dataset. It is noteworthy that the Bézier-GAN achieved comparable performance in compression capability and reconstruction error to linear PCA, indicating that the underlying probability distribution of the data may be relatively simple, such as a Gaussian. 

\section{Main Contribution}
Our research reveals that traditional PCA techniques employed in prior studies like those discussed in \cite{ghoman2012pod} and \cite{diez2015-CMAME}, may generate geometrically feasible shape modifications for optimization, but have anomalous geometric features such as sharp angles and deflections with the considerable disadvantage of wasting expensive computer simulations when performed with respect to those suboptimal geometries.

To solve this drawback, we propose a new optimization model where anomalous geometries produced from the reduced representation (or \text{latent space}) are penalized and consequently avoided during the optimization loop. In this setting, we propose a new SBDO framework based on probabilistic linear latent variable models such as FA 
and PPCA for dimensionality reduction and density estimation. 

The main assumption in this paper is that the data is generated by a Gaussian distribution since the two models are part of the linear Gaussian models family \cite{roweis1999unifying}.
The Gaussian assumption is satisfied when the geometrical modification is linear with respect to the design variables such in Free Form Deformation \cite{sederberg1986-ACM}, NURBS \cite{piegl1987curve,piegl1996nurbs}, radial basis function \cite{sieger2014rbf} and PARSEC \cite{sobieczky1999parametric}. 
In this case, the geometries collected in the dataset are generated by a linear combination of uniformly distributed design variables so that the original design space PDF, follows approximately a Gaussian distribution as a simple and direct application of the central limit theorem (CLT).

Once the parameters of the Gaussian distribution are estimated and the reduced subspace assessed through PPCA/FA, we can compute the anomaly score with respect to any new geometry generated from the latent space. In particular, we represent the anomaly score using the Mahalanobis distance \cite{mahalanobis1936generalized}.
In our new framework, once a geometry is generated from the latent space during the optimization, we retrieve its Mahalanobis distance. A high Mahalanobis distance indicates that the current geometry might be anomalous or not belong to the original parametrization and hence penalized from the new optimization model.

To assess the performance of this approach, two global optimization algorithms namely DIRECT \cite{jones1993-JOTA} and Bayesian optimization \cite{mockus1978application} based on Gaussian process regression \cite{williams2006gaussian} with lower confidence bound \cite{cox1992statistical} as acquisition function (GP-LCB) are applied for a hull form shape optimization of the DTMB 5415 model an early and open to a public version of the USS Arleigh Burke destroyer DDG 51, extensively used as an international benchmark for shape optimization problems (e.g., \cite{diez2018-AVT204-3,diez2018-AVT252-7}). 
From the numerical results, our framework facilitates the convergence of global optimization algorithms towards better minima, and at the same time, anomalous geometries are avoided and penalized during the optimization loop so that precious computationally expensive simulations are not wasted.

The paper is structured in the following manner: in section \ref{sec:4}, we introduce the three primary components of the SBDO framework, namely, the optimization model, shape parametrization method, and physics simulator. Moving on to section \ref{sec:5}, we delve into the dataset generation and application of dimensionality reduction using PCA. 
In section \ref{sec:6}, we present our framework based on generative models including PPCA and FA for density estimation and dimensionality reduction. Furthermore, we present a novel anomaly-aware optimization model, highlighting the statistical properties of the design space and providing a geometrical interpretation of our framework.  
In section \ref{sec:7}, we apply our methodology for hull form shape optimization comparing it to classical PCA. Finally, we summarize our findings and potential future work in section \ref{sec:8}.
\section{The Simulation-Based Design Optimization Framework}\label{sec:4}
\begin{figure}[htb!]\label{fig:sbdo1}
	\centering
	\includegraphics[width=0.5\textwidth]{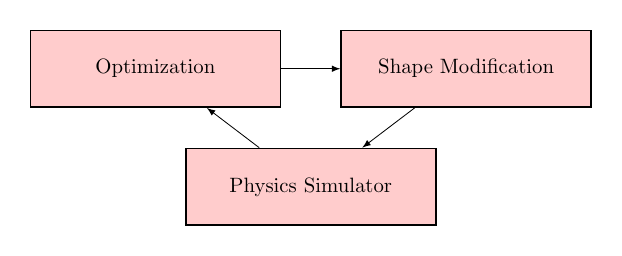}
	\caption{Scheme for the SBDO framework.}\label{fig:SBDOframe}
\end{figure}
In the following sections, we'll describe the main components of the SBDO framework as shown in Fig. \ref{fig:SBDOframe}. 
\subsection{The Optimization Problem}
The first block of the SBDO framework is the optimization process. Generally, the underlying optimization problem that we want to solve, given the design variable $\mathbf{v}\in \mathbb{R}^M$ is the following
\begin{alignat}{3}
	\min_{\mathbf{v}}  & \quad           f(\mathbf{x}(\mathbf{v})) \label{eq:opt_sbdo1} \\ 
	&  \quad g_j(\mathbf{x}(\mathbf{v}))  \leq a_j \label{eq:opt_sbdo2} &\quad  j=1,\dots,J \\
	&\quad  v_{m}^{\text{lb}} \leq  v_{m} \leq  v_{m}^{\text{ub}} \label{eq:opt_sbdo3}&\qquad  m = 1,\dots, M    
\end{alignat}
here the objective function $f : \mathbb{R}^D \rightarrow \mathbb{R}$ in Eq. \ref{eq:opt_sbdo1} represents the quantity that we want to minimize. With the vector $\mathbf{x}(\mathbf{v})$ we represent a geometry $\mathbf{x} \in \mathbb{R}^D$ where its modification depends on the design variable vector $\mathbf{v}$. The optimization is performed with respect to the design variable  $\mathbf{v}$ since it's responsible for the geometrical modification. The function evaluation is usually performed by a physics simulator.
In this context, black box optimization problems are solved with global optimization algorithms and especially when the gradient is supposed to be noisy, a derivative-free approach is more suitable. 

Eq. \ref{eq:opt_sbdo2} and Eq. \ref{eq:opt_sbdo3} represent the constraints. The function $g(\mathbf{x}(\mathbf{v}))$ represents a geometrical constraint that forces the optimizer to produce admissible shape modifications. The geometrical constraints are evaluated before the function evaluation is performed. Constraints of this kind are called \textit{hidden constraints} because they are not specified to the simulator \cite{digabel2015taxonomy}.

Finally, we have to define the upper bound $v_{m}^{\text{lb}}$ and the lower bound $v_{m}^{\text{ub}}$ for each $m$th component of the design variable $\mathbf{v}$. The bounds are very important for the overall SBDO process because an eventually large value of the lower and upper bounds allows the optimizer to search on a larger space, making the optimization more challenging with the possibility of violating many times the geometrical constraints. On the other side, with a larger design space, we could obtain more different configurations of the design variables with the chance to produce a better improvement in the objective function.
\subsection{Shape Modification: The Free Form Deformation Method}
The second block of the SBDO framework is shape parameterization. This comes after the optimization block because the shape parameterization method will receive the design variable vector $\mathbf{v}$ from the optimizer and will produce the relative shape modification to the geometry $\mathbf{x}(\mathbf{v})$. 
The choice of the shape parameterization technique has a large impact on the practical implementation and the success of the optimization process. 

Here, we show one of the most popular methods for shape modifications, namely the FFD. 
The idea is to embed an object within a trapezoidal (or other topology) lattice and modify the object within the trapezoid as the lattice is modified. A local coordinate system is assumed, with origin $\mathbf{\mathsf{r}}_0 \in \mathbb{R}^3$ with $\mathbf{\mathsf{r}}_0 = (\mathsf{x}_0, \mathsf{y}_0, \mathsf{z}_0)$ at one of the trapezoid vertices. Any point within the trapezoid has $\alpha$, $\beta$, and $\gamma$ coordinates such that
\begin{equation}\label{eq:localffd}
	\mathbf{\mathsf{r}} = \mathbf{\mathsf{r}}_0 + \alpha\hat{\mathbf{T}}_1 + \beta\hat{\mathbf{T}}_2 + \gamma\hat{\mathbf{T}}_3
\end{equation}
with $\alpha$, $\beta$, and $\gamma$ bounded by $[0,1]$ and given by
\begin{equation}\label{eq:localcoord1}\nonumber
	\alpha = \frac{\hat{\mathbf{T}}_2\times\hat{\mathbf{T}}_3 \cdot(\mathbf{\mathsf{r}} - \mathbf{\mathsf{r}}_0 )}{\hat{\mathbf{T}}_2\times\hat{\mathbf{T}}_3\cdot\hat{\mathbf{T}}_1}, 
\end{equation}
\begin{equation}\label{eq:localcoord2}\nonumber
	\beta  = \frac{\hat{\mathbf{T}}_1\times\hat{\mathbf{T}}_3 \cdot(\mathbf{\mathsf{r}} - \mathbf{\mathsf{r}}_0 )}{\hat{\mathbf{T}}_1\times\hat{\mathbf{T}}_3\cdot\hat{\mathbf{T}}_2},
\end{equation}
\begin{equation}\label{eq:localcoord3}
	\gamma = \frac{\hat{\mathbf{T}}_1\times\hat{\mathbf{T}}_2 \cdot(\mathbf{\mathsf{r}} - \mathbf{\mathsf{r}}_0 )}{\hat{\mathbf{T}}_1\times\hat{\mathbf{T}}_2\cdot\hat{\mathbf{T}}_3} \,\,\,\,\,
\end{equation}
Control points (CPs) $\mathbf{c}_{ijk} \in \mathbb{R}^{3}$ are defined as lattice nodes. The number of CPs used in  $\hat{\mathbf{T}}_1$, $\hat{\mathbf{T}}_2$, and $\hat{\mathbf{T}}_3$ directions are $t_1$, $t_2$, and $t_3$, respectively with a total number of CPs equals to $t_{\text{tot}} = t_1 + t_2+ t_3$. The coordinates of modified CPs depend on the imposed original-lattice nodes and to perform a modification to the geometry the coordinates of the CPs are perturbed by the relative \textit{design variable} vector $\mathbf{v}_{ijk} \in \mathbb{R}^{3}$, as
\begin{equation}\label{eq:design_ffd}
	c_{ijk}(\mathbf{v}_{ijk}) =  \mathbf{\mathsf{r}}_0  + \frac{i}{t_1}\hat{\mathbf{T}}_1 + \frac{j}{t_2}\hat{\mathbf{T}}_2 + \frac{k}{t_3}\hat{\mathbf{T}}_3 + v_{ijk}
\end{equation}
The shape modification is achieved by interpolating the CPs' modification over the embedding space. The interpolation can be performed using different polynomial bases. Herein, a tensor product of trivariate Bernstein polynomial is used \cite{sederberg1986-ACM}, where the discretized geometry $\mathbf{x(\mathbf{v}})$ is given by
\begin{equation}\label{eq:ffd_bern}
	\begin{split}
		\mathbf{x}(\mathbf{v}) &= \mathbf{g}_0 + \sum_{i=0}^{t_1} \sum_{j=0}^{t_2} \sum_{k=0}^{t_3}b_{i, t_1}(\alpha)b_{j, t_2}(\beta)b_{k, t_3}(\gamma)c_{ijk}(v_{ijk})\\
		&= \mathbf{g}_0 + \sum_{i=0}^{t_1} \sum_{j=0}^{t_2} \sum_{k=0}^{t_3}c_{ijk}(v_{ijk}) B_{i,j,k}(\alpha, \beta, \gamma)
	\end{split}
\end{equation}
where $\mathbf{x(\mathbf{v})}$ is a vector containing the Cartesian coordinates of the displaced points, and $\mathbf{g}_0$ represents the original geometry.
The generic Bernstein basis polynomials are defined as
\begin{equation}
	b_{v,r}(\chi)=\binom{r}{v}\chi^v(1-\chi)^{r-v}
\end{equation}
Note that from now on the design variable vector $\mathbf{v}$ is considered as a $M$-dimensional vector where $M = 3t_{\text{tot}}$ and the geometry $\mathbf{x(\mathbf{v})}$ as a $D$-dimensional vector. 
\subsection{Physical Solver}
In general, a physical solver is a computer program that provides an approximation of the behavior of a physical system determined by its governing equations and boundary conditions. Computer simulations are crucial nowadays for scientific research in many domains, from fluid dynamics to material science, chemistry, and biology. Especially when direct experimentation is too economically expensive, dangerous, or even impossible to perform. In fluid dynamics, the governing equations are given by the Navier-Stokes equations that we briefly introduce in this section.
Without considering the energy in the system, the continuity equations and the Cauchy momentum equation are
\begin{equation}
	\frac{\partial \rho}{\partial t} + \nabla \cdot (\rho \mathbf{\mathsf{u}}) = 0 
\end{equation}
\begin{equation}
	\rho\frac{D \mathsf{u}}{D t} = \nabla \cdot \mathbf{Y} + \rho \mathbf{f}
\end{equation}
where $\mathsf{u}$ is the velocity vector, $\rho$ is the density, $\mathbf{Y}$ is the second order Cauchy stress tensor, and $\rho \mathbf{f}$ is the volume forces vector. With $\mathbf{Y} =  p \mathbf{I} + \tau$, for a Newtonian and incompressible fluid the Navier-Stokes equation is given by
\begin{equation}
	\rho \frac{\partial \mathbf{\mathsf{u}}}{\partial t} + \rho (\mathbf{\mathsf{u}} \cdot \nabla) \mathbf{\mathsf{u}} = \nabla \cdot [- p \mathbf{I}+ \eta [\nabla \mathbf{\mathsf{u}} + (\nabla \mathbf{\mathsf{u}})^\top]] + \rho \mathbf{f} 
\end{equation}
assuming that the stress tensor is a linear function of the
strain tensor, the fluid is isotropic, $\nabla \cdot \tau = 0$ for a fluid at rest, and $\eta$ and $p$ are the dynamic viscosity of the fluid and the pressure respectively.

The Navier-Stokes completely describe the dynamics of a fluid as the turbulence described by the nonlinear term. In many real-world applications,
the turbulence effects must be considered to provide an accurate description of the physical process. Still, nowadays, the computational effort required for the direct numerical simulation (DNS) of the Navier-Stokes equations is intractable. Most of the time simplified numerical models are considered such as the large eddy simulation (LES), Reynolds-averaged Navier-Stokes equations (RANS), and potential flow simulators.
\section{Design-Space Dimensionality Reduction for SBDO}\label{sec:5}
The curse of dimensionality is one of the major drawbacks when global optimization algorithms are employed. In the following sections, we show how to reduce the dimensionality of the design vector $\mathbf{v} \in  \mathbb{R}^M$ by learning a new parameterization for shape modification. More precisely the intent is to perform the optimization with respect to a new design variable vector $\mathbf{z} \in \mathbb{R}^{K}$ with $K < M$ to improve the convergence speed of the global optimization routine to an optimal solution. 

The process is highlighted in Fig. \ref{fig:sbdo_dr}. The first phase is to generate the dataset $\mathbf{X}$ and this is done by randomly sampling the design variable vector $\mathbf{v}$ from a uniform distribution. As support for the uniform distribution, we use the upper bounds and lower bounds of each design variable component. The second phase is to apply a dimensionality reduction model which learns the new parameterization given by the new design variable vector $\mathbf{z}$ and a matrix $\mathbf{U}$. Here the matrix $\mathbf{U}$ is responsible for transforming (or decoding) the design variable $\mathbf{z}$ in a new geometry $\mathbf{x(\mathbf{z})}$. The matrix $\mathbf{U}$ replaces the shape parameterization method (e.g. the FFD method) during the SBDO loop.  
\begin{figure}[htb!]
	\centering
	\includegraphics[width=0.8\textwidth]{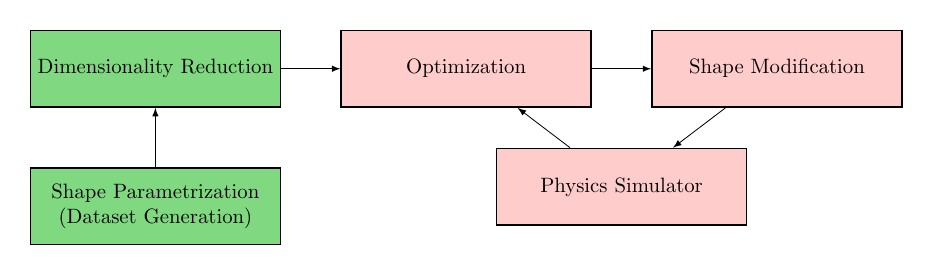}
	\caption{SBDO framework with a design-space dimensionality reduction phase.}\label{fig:sbdo_dr}
\end{figure}
\subsection{Dataset Generation}\label{sec:genFFD}
The first step in the DR-SBDO is to generate the dataset used in the next block to train the dimensionality reduction model. This can be simply achieved by sampling from a uniform distribution the full dimensional design variable vector 
\begin{equation}
 v_m \sim \mathcal{U}(v_{m}^{\text{lb}}, v_{m}^{\text{ub}}) \qquad  m = 1, \dots, M 
\end{equation}
and using the sampled vector as input for Eq. \ref{eq:ffd_bern} to produce a random geometry. We repeat the process $N$ times so we can collect all the geometries inside a dataset $\mathbf{X}$ of size ($N\times D$), where $D$ is given by the product between the number of the grids points $L$ and the number of Cartesian coordinate ($\mathsf{x}, \mathsf{y}, \mathsf{z}$) considered in the application.

In this step, it is crucial that the designer carefully selects the value for the upper and lower bounds of the design variables. The first reason is that a restricted range for the design variables will not allow diversity but redundancy (i.e. low variance) in the generated dataset. A large range for the design variables might produce, also depending on the shape parametrization method, many geometrically anomalous geometries and furthermore without satisfying the geometrical constraints.
Once the dataset is generated, we are ready to perform the dimensionality reduction step.
\subsection{Principal component analysis}
Principal component analysis (PCA) has become during the last decades one of the most famous statistical tools for feature extraction and dimensionality reduction. It appeared for the first time in a paper \cite{pearson1901liii} published by Karl Pearson in 1901 describing the model from a geometrical perspective and then independently developed from a statistical point of view by Harold Hotelling in the early '30's \cite{hotelling1933analysis}. 

The PCA finds a new optimal basis of dimension $K$ such that the variance of the projected points is maximized and the mean squared error between the original data and their projection is minimized. The first step in the PCA is the computation of the sample mean vector $\bar{\mathbf{x}} = \frac{1}{N}\sum_{n=1}^{N} \mathbf{x}_n $ and the computation of the sample covariance matrix $\mathbf{S} = \frac{1}{N}\sum_{n=1}^{N} (\mathbf{x}_n - \bar{\mathbf{x}})(\mathbf{x}_n - \bar{\mathbf{x}})^\top$. The second step involves an eigendecomposition of the matrix $\mathbf{S}$ 
\begin{equation}\label{eq:eigproblem_pca}
	\mathbf{S}\mathbf{u} = \mathbf{\lambda}\mathbf{u}
\end{equation}
of size ($D\times D$). To reduce the dimensionality of our dataset $\mathbf{X}$ we should choose a subset of $K$ eigenvectors which corresponds to the top-$K$ largest variance eigenvalues. Then we can project a data vector $\mathbf{x}$ into the subspace defined by top-$K$ orthonormal eigenvectors $\mathbf{U}$ (of size $D\times K$) also called \textit{principal components} of the sample covariance matrix $\mathbf{S}$
\begin{equation}\label{eq:proj_pca}
	\mathbf{z}_{n} = \mathbf{U}^{\top}(\mathbf{x}_{n} - \bar{\mathbf{x}})
\end{equation}
the vector $\mathbf{z}_{n}$ of size $K$ represents the reduced representation of the original data $\mathbf{x}_{n}$. Usually the random variable $\mathbf{z}$ 
is called \textit{latent variable} because they are not directly observed but they explain or extract some useful patterns of the observed data $\mathbf{x}$. The matrix $\mathbf{Z}$ of size ($N\times K$) collects all the latent variables $\mathbf{z}_{n}$.
In case is needed, the reconstruction of the relative data point can be obtained by projecting back to data space $\mathbb{R}^{D}$ as follows 
\begin{equation}\label{eq:proj_back_pca}
	\tilde{\mathbf{x}}_{n} = \mathbf{U}\mathbf{z}_{n} + \bar{\mathbf{x}} = \mathbf{U}\mathbf{U}^{\top}\mathbf{x}_{n}  + \bar{\mathbf{x}} = \mathbf{P}\mathbf{x}_{n}  + \bar{\mathbf{x}}
\end{equation}
where $\mathbf{P} = \mathbf{U}\mathbf{U}^{\top}$ is the symmetric projection matrix for the subspace spanned by the first top-$K$ eigenvectors $\mathbf{U}$.

Practical implementations of the PCA perform the singular value decomposition \cite{golub1970-NM} of the data matrix $\mathbf{X} = \mathbf{V}\boldsymbol{\Xi}\mathbf{U}^{\top}$, where $\mathbf{V}$, $\boldsymbol{\Xi}$ and $\mathbf{U}$ are ($N\times K$), ($K\times K$) and ($K\times D$) respectively, with $K \leq \min(N, D)$.  In this case $\mathbf{X}$ must be centered before the SVD computation to obtain the correct results. The sample covariance matrix in this setting is given by $\mathbf{S} =\frac{1}{N} \mathbf{X}^{\top}\mathbf{X}$ and consequently
\begin{equation}\label{eq:svd_cov}
	\mathbf{S} =\frac{1}{N} \mathbf{X}^{\top}\mathbf{X} = \frac{1}{N} \mathbf{U}\boldsymbol{\Xi}\mathbf{V}^{\top}\mathbf{V}\boldsymbol{\Xi}\mathbf{U}^{\top} = \frac{1}{N} \mathbf{U}\boldsymbol{\Xi}^{2}\mathbf{U}^{\top} = \mathbf{U}\boldsymbol{\Lambda}\mathbf{U}^{\top}
\end{equation}
the eigenvalues of the covariance matrix are given by $\lambda_i = \frac{\xi^{2}_i}{N}$, where $\xi_i$ represents the $i$th element in the diagonal of the matrix $\boldsymbol{\Xi}$.

A simple and widely used criterion for choosing the number of components $K$ to retain is given by computing the fraction explained variance with respect to the total variance given by 
\begin{equation}
\frac{\sum_{k = 1}^{K}\lambda_k}{\Tr{(S)}}    
\end{equation}
and setting $K$ to the desired level.
\subsection{Optimization in the Latent Space}
Once we trained the PCA model with respect to the dataset $\mathbf{X}$, we can formulate the new optimization model
\begin{alignat}{3}
	\min_{\mathbf{z}}   & \quad  f(\mathbf{x}(\mathbf{z})) \label{eq:ropt_sbdo1} \\ 
	&  \quad g_j(\mathbf{x}(\mathbf{z}))  \leq a_j \label{eq:ropt_sbdo2} &\quad  j=1,\dots,J \\
	&\quad  z_{k}^{\text{lb}} \leq  z_{k} \leq  z_{k}^{\text{ub}} \label{eq:ropt_sbdo3} &\qquad  k = 1,\dots, K
\end{alignat}
where the modification of the geometry $\mathbf{x}$ is performed by the variable $\mathbf{z} \in \mathbb{R}^{K}$ and the bounds are fixed taking the maximum and the minimum of each column component $z_{k}$ of $\mathbf{Z}$. 
Therefore, the optimizer iterates in the latent space using the variable $\mathbf{z}$ and the dimensionality of this subspace is smaller than the full dimensionality space defined by $\mathbf{v}\in \mathbb{R}^{M}$. This reduces the dimensionality of the optimization problem, which can improve the performance of the global optimization algorithm in finding the basin of attraction of the global minimum more efficiently.
\subsection{Decoding from the Latent Space}
After each iteration in the reduced dimensionality space, the optimizer outputs a new design variable $\mathbf{z}$ in the hope that it will lead to a significant improvement in the objective function $f(\mathbf{x}(\mathbf{z}))$. However, to calculate the objective function value, the design variable vector $\mathbf{z}$ must be transformed back into the data space, which requires projecting it onto the modified geometry $\mathbf{x}(\mathbf{z})$ using the principal components of the data $\mathbf{U}$. This projection can be easily accomplished through a matrix multiplication of $\mathbf{U}$ and $\mathbf{z}$, along with the addition of the mean vector $\mathbf{\bar{\mathbf{x}}}$. The resulting equation is
\begin{equation}
	\mathbf{x} = \mathbf{U}\mathbf{z} + \mathbf{\bar{\mathbf{x}}}
\end{equation}
where to maintain the notation uncluttered we suppose that $\mathbf{x}(\mathbf{z}) = \mathbf{x}$. The eigenvectors $\mathbf{U}$ represent spatial geometrical components and are responsible for decoding back the design variable $\mathbf{z}$ in the space $\mathbf{\mathbb{R}}^{D}$. 
Once we obtain the geometry from the latent space $\mathbf{x}$, this will enter in the physical solver that computes the objective function value as shown in Fig. \ref{fig:sbdo_dr}.
\section{A Novel SBDO Framework}\label{sec:6}
In this section, we introduce some important modifications to the framework described before. 
The SBDO framework with a dimensionality reduction phase performed before the optimization routine allows the optimizer to iterate in a subspace of lower dimensionality of $\mathbb{R}^{K}$ with respect to the full dimensionality space of $\mathbb{R}^{M}$ with $K<M$. An important property that could be highly desirable when performing the optimization in the latent space is that the shape modification should be \textit{invariant} respect the full dimensionality design space $\mathbb{R}^{M}$. This means that the design variable vector $\mathbf{z}$ should not produce geometries that the original full dimensional design variable vector $\mathbf{v}$ is not able to generate. 
In the next sections, we describe more in detail the problems and how to solve them using the new framework for shape optimization, precisely we discuss
\begin{itemize}
	\item In section \ref{sec:statsFFD} we show that in case the design variables are sampled from a uniform distribution, the geometries produced with the shape parameterization method (e.g. FFD) produce a design space that approximately follows a Gaussian distribution as a direct application of the central limit theorem (CLT).
	\item In section \ref{sec:FA} and \ref{sec:PPCA} we describe two probabilistic latent variable models, factor analysis and probabilistic PCA for dimensionality reduction and density estimation for the new SBDO framework. 
	\item  In section \ref{sec:mahopt}, we introduce a novel optimization model that incorporates anomaly detection through the computation of the Mahalanobis distance. We will also analyze the nature of the distribution of the Mahalanobis distance for normal random variables and illustrate the significance of the new optimization model by examining it from a geometrical perspective.
\end{itemize}
\subsection{Statistical Properties of the Shape Parametrization Method}\label{sec:statsFFD}
In section \ref{sec:genFFD}, we presented the FFD technique, which computes a geometry by evaluating the Eq. \ref{eq:ffd_bern}. This equation represents a linear combination of fixed Bernstein polynomials of the Cartesian coordinates, along with a coefficient (i.e. the design variable) that introduces a random perturbation to the geometry around those Cartesian coordinates.
Consequently in discrete form, Eq. \ref{eq:ffd_bern} represents a sum of random vectors $\mathbf{h} \in \mathbb{R}^D$ dependent of the design variable vector $\mathbf{v}$
\begin{equation}
	\mathbf{x(\mathbf{v})} = \sum_{m = 1}^{M} \mathbf{h}(v_m)
\end{equation}

Exploring the statistical properties of $\mathbf{x(\mathbf{v})}$ could be worthwhile, particularly in determining the probability distribution from which the data is generated.  
The central limit theorem (CLT) \cite{tchebycheff1889residus} shows that under certain conditions, the probability distribution of the sum of a large number of independently and identically distributed random variables converges to a Gaussian distribution. More formally
\begin{theorem}
	For a sequence of $D$ dimensional random vectors $\mathbf{h}(v_m)$ with finite mean and covariance $\boldsymbol{\mu}$ and $\boldsymbol{\Sigma}$ respectively, we have that
	\begin{equation}
		\sqrt{M}\left(\frac{1}{M} \sum_{m=1}^{M} \mathbf{h}(v_m) - \boldsymbol{\mu}\right)  \xrightarrow{d} \mathcal{N}(0, \boldsymbol{\Sigma})
	\end{equation}
	as $M \rightarrow \infty$.
\end{theorem}
which implies that $\sum_{m=1}^{M} \mathbf{h}(v_m) \xrightarrow{d} \mathcal{N}(M\boldsymbol{\mu}, M\boldsymbol{\Sigma})$ for $M \rightarrow \infty$. The proof is available for example in \cite{van2000asymptotic}. 

Empirically, is well known that when a random variable $\mathbf{h}$ is generated by uniform distribution, the sum converges to a Gaussian distribution very fast (i.e. for $M << \infty$).
This result is significant because if the probability distribution of the data $\mathbf{X}$ is approximately Gaussian, it can be used to guide the optimizer to remain in the region defined by the hyperellipsoid of the Gaussian distribution.

It is worth noting that other popular shape parametrization methods rely on a linear combination of the design variables to perform a shape modification such as NURBS \cite{piegl1987curve,piegl1996nurbs}, radial basis functions \cite{sieger2014rbf} and PARSEC \cite{sobieczky1999parametric}. 
This suggests that the methodology proposed in this paper may apply to a range of shape parametrization methods.

Finally, it is important to recognize that the geometries $\mathbf{x}\in\mathbb{R}^{D}$ are produced from a design variable vector of much lower dimensionality $\mathbf{v}\in\mathbb{R}^{M}$, with $M << D$. 
Consequently, the degrees of freedom of the data $\mathbf{x}$ are much lower than $D$, and the data may be assumed to lie on a lower dimensional linear manifold. 
As a result, the sample covariance matrix $\mathbf{S}$ is expected to have many zero eigenvalues and may not be invertible.

In the next section, we describe two probabilistic latent variable models that we can use for density estimation and dimensionality reduction in the new framework.
\subsection{Factor Analysis}\label{sec:FA}
Factor analysis \cite{bartholomew1984foundations, spearman1961general, cattell1965biometrics} is one of the simplest and most fundamental generative models \cite{ghojogh2023elements} that during its long history has been applied in many different fields. 
We start from the assumption that the variable $\mathbf{x}$ can be written as a linear combination of a latent variable $\mathbf{z}$
\begin{equation}\label{eq:xz_fa_conmean}
	\mathbf{x} = \mathbf{W}\mathbf{z} + \boldsymbol{\mu} + \boldsymbol{\epsilon}
\end{equation}
and a \textit{factor loading} matrix $\mathbf{W}$ of size $(D\times K)$. 
The latent variable is generated by a standard Gaussian distribution $\mathbf{z}\sim \mathcal{N}(0, \mathbf{I})$ like the noise term $\boldsymbol{\epsilon} \sim \mathcal{N}(0, \boldsymbol{\Psi})$, independent from $\mathbf{z}$.
We can write that the conditional distribution is given by 
\begin{equation}\label{eq:xz_fa}
	p(\mathbf{x}|\mathbf{z}) = \mathcal{N}(\mathbf{x}|\mathbf{W}\mathbf{z} + \boldsymbol{\mu}, \boldsymbol{\Psi})
\end{equation}
a Gaussian distribution with conditional mean 
\begin{equation}\label{eq:decodeFA1}
\mathbb{E}[\mathbf{x}|\mathbf{z}] = \mathbf{W}\mathbf{z} + \boldsymbol{\mu}  
\end{equation}
and a diagonal conditional covariance ($D\times D$) matrix $\boldsymbol{\Psi}$. The FA model is described by the parameters $\boldsymbol{\Theta} = \{\mathbf{W}, \boldsymbol{\Psi}\}$  for a total of $(D\times K + D)$ parameters. 
We find the expression for the marginal distribution $p(\mathbf{x})$ solving the following integral
\begin{equation}\label{eq:marginal_fa}
	p(\mathbf{x}) = \int_\mathbf{z} p(\mathbf{x}|\mathbf{z}) p(\mathbf{z}) \,d\mathbf{z}
\end{equation}
where Eq. \ref{eq:marginal_fa} can be computed in closed form because it represents a convolution of two Gaussian distributions. The marginal is again Gaussian, $p(\mathbf{x}) = \mathcal{N}(\mathbf{x}| \boldsymbol{\mu}, \mathbf{C})$ with the covariance matrix equal to
\begin{equation}\label{eq:cov_fa}
	\mathbf{C} = \mathbf{W}\mathbf{W}^\top + \boldsymbol{\Psi}
\end{equation} 
The diagonal elements of $\mathbf{C}$ are given by a sum of two terms. The first term, $||\mathbf{w}_i||^{2}$, is called \textit{communality} because it represents the variance explained by the $K$ factors $\mathbf{W}$ respect to the feature $x_i$. The second term is the variance relative to the feature $x_i$ that is not explained by the factors and is called \textit{uniqueness}.
Given those results, we can write down the expression for the posterior distribution $p(\mathbf{z}|\mathbf{x})$ 
\begin{equation}\label{eq:zx_fa}
	p(\mathbf{z}|\mathbf{x}) = \mathcal{N}(\mathbf{z}|\mathbf{W}^{\top}\mathbf{C}^{-1}(\mathbf{x} - \boldsymbol{\mu}), \mathbf{G})
\end{equation}
with the matrix $\mathbf{G}$ equals to 
\begin{equation}\label{eq:mat_G}
	\mathbf{G} = (\mathbf{W}^\top\boldsymbol{\Psi}\mathbf{W} +\mathbf{I})^{-1}
\end{equation} 
and the posterior mean defined as a linear function of $\mathbf{x}$, as shown below
\begin{equation}\label{eq:encodeFA}
	\mathbb{E}[\mathbf{z}|\mathbf{x}] = \mathbf{W}^\top\mathbf{C}^{-1}(\mathbf{x} - \boldsymbol{\mu})
\end{equation}
while the posterior covariance $\mathbf{G}$ is independent from $\mathbf{x}$. 
The posterior mean can be reconstructed in data space using 
\begin{equation}\label{eq:decodeFA2}
	\tilde{\mathbf{x}} = \mathbf{W}\mathbb{E}[\mathbf{z}|\mathbf{x}] + \boldsymbol{\mu}    
\end{equation}

We can find the set of parameters $\boldsymbol{\Theta} = \{\mathbf{W}, \boldsymbol{\Psi}\}$ via the classical maximum likelihood estimation (MLE) with respect to the marginal likelihood using the information in our observed data $\mathbf{x}_n \in \mathbf{X}$. 
Unfortunately, for the FA model, a closed-form solution of the marginal likelihood is not available, necessitating the use of numerical optimization methods to find the parameters.
One efficient and viable method for finding parameters in latent variable models is to use the expectation maximization (EM) algorithm \cite{dempster1977maximum}. 
The EM algorithm aims to maximize the expectation of the joint log-likelihood $\ln p(\mathbf{X}, \mathbf{Z})$ respect to the posterior $p(\mathbf{z}|\mathbf{x})$. We can compute  
\begin{equation}
	\ln p(\mathbf{X}, \mathbf{Z}) = \sum_{n=1}^{N}\ln p(\mathbf{x}_n|\mathbf{z}_n) + \ln p(\mathbf{z}_n)
\end{equation}
then taking the expectation respect to the posterior $p(\mathbf{z}|\mathbf{x})$ we obtain
\begin{equation}
	\begin{split}\label{eq:max_like_fa}
		\mathbb{E}_{\mathbf{z}|\mathbf{x}}[ \ln p(\mathbf{X},\mathbf{Z}| \boldsymbol{\Theta})] 
		&= \sum_{n=1}^{N} \mathbb{E}_{\mathbf{z}|\mathbf{x}}[\ln p(\mathbf{x}_n|\mathbf{z}_n, \boldsymbol{\Theta}) + \ln p(\mathbf{z}_n)] \\
		&= \sum_{n=1}^{N} \mathbb{E}_{\mathbf{z}|\mathbf{x}}\bigg[-\frac{D}{2} \ln(2\pi) - \frac{1}{2}\ln |\boldsymbol{\Psi}| - \\ & 
		\qquad \frac{1}{2}(\mathbf{x}_n - \mathbf{W}\mathbf{z}_n - \boldsymbol{\mu})^\top \boldsymbol{\Psi}^{-1} (\mathbf{x}_n - \mathbf{W}\mathbf{z}_n - \boldsymbol{\mu})\bigg]
	\end{split}
\end{equation}
where we omitted the prior distribution since it does not depend on the set of parameters $\boldsymbol{\Theta}$. In the EM algorithm, the following two steps are repeated until convergence 
\begin{itemize}
	\item Expectation step (E-step) where we compute the sufficient statistics of the posterior distribution with respect to the old parameters $p(\mathbf{z}_n|\mathbf{x}_n, \boldsymbol{\Theta}_{\text{old}})$ 
	\begin{gather}\label{E_step_fa}
		\mathbb{E}_{\mathbf{z}|\mathbf{x}}[\mathbf{z}_n] =\mathbf{G} \mathbf{W}^{\top} \boldsymbol{\Psi}^{-1}(\mathbf{x}_n - \bar{\mathbf{x}}) \\
		\mathbb{E}_{\mathbf{z}|\mathbf{x}}[\mathbf{z}_n \mathbf{z}_n^{\top}] = \mathbf{G} + \mathbb{E}_{\mathbf{z}|\mathbf{x}}[\mathbf{z}_n]\mathbb{E}_{\mathbf{z}|\mathbf{x}}[\mathbf{z}_n]^{\top}
	\end{gather}
	\item Then given the statistics from the E-step, in the new Maximization step (M-step) we evaluate the parameters maximizing $\mathbb{E}_{\mathbf{z}|\mathbf{x}}[\ln p(\mathbf{X}, \mathbf{Z}| \boldsymbol{\Theta_{\text{old}}})]$ setting its derivatives respect to the parameters $\boldsymbol{\Theta}$ to zero
	\begin{gather}\label{M_step_fa}
		\mathbf{W}_{\text{new}} = \bigg[\sum_{n=1}^{N}(\mathbf{x}_n - \bar{\mathbf{x}})\mathbb{E}_{\mathbf{z}|\mathbf{x}}[\mathbf{z}_n]^{\top}\bigg]\bigg[\sum_{n=1}^{N}\mathbb{E}_{\mathbf{z}|\mathbf{x}}[\mathbf{z}_n \mathbf{z}_n^{\top}]\bigg]^{-1} \\
		\mathbf{\Psi}_{\text{new}} = \text{diag}\biggl\{\mathbf{S} - \mathbf{W}_{\text{new}} \frac{1}{N} \sum_{n=1}^{N} \mathbb{E}_{\mathbf{z}|\mathbf{x}}[\mathbf{z}_n](\mathbf{x}_n- \bar{\mathbf{x}})^{\top} \biggr\}
	\end{gather}
\end{itemize} 
The iterative procedure described above efficiently produces at the end of the iterations, a stationary point. 

Suppose now that $\mathbf{R}$ is a ($D \times D$) orthogonal matrix if we define the new rotated factor loadings 
\begin{equation}
\tilde{\mathbf{W}} = \mathbf{W}\mathbf{R}    
\end{equation}
the marginal covariance in Eq. \ref{eq:cov_fa} with this new reparametrization is given by 
\begin{equation}\label{eq:fa_R}
	\mathbf{C} = \tilde{\mathbf{W}}\tilde{\mathbf{W}}^\top +\boldsymbol{\Psi} = \mathbf{\mathbf{W}}\mathbf{R}\mathbf{R}^\top\mathbf{\mathbf{W}} +\boldsymbol{\Psi} = \mathbf{W}\mathbf{W}^\top +\boldsymbol{\Psi}
\end{equation}
which means that performing a rotation of the latent space achieves the same value for the marginal distribution. Consequently, the matrix $\mathbf{W}$ is not uniquely identifiable. 

A remark regards the computation of the inverse of the covariance matrix $\mathbf{C}$. Instead of directly take the inverse of $\mathbf{C}$, we can use the Woodbury matrix inversion formula \cite{woodbury1950inverting} obtaining
\begin{equation}\label{eq:C_inv_FA}
	\mathbf{C}^{-1} = \boldsymbol{\Psi}^{-1} - \boldsymbol{\Psi}^{-1}\mathbf{W}\mathbf{G}\mathbf{W}^\top\boldsymbol{\Psi}^{-1}
\end{equation}
with $\mathbf{G}$ defined in Eq. \ref{eq:mat_G}. Using the previous transformation we only need to compute the matrix $\mathbf{G}$ which is ($K\times K$) rather than find directly the inverse of $\mathbf{C}$ with time complexity $\mathcal{O}(D^3)$.

In the next section, we discuss a special case of FA where the conditional covariance matrix $\boldsymbol{\Psi}$ is restricted to be isotropic allowing the computation of the MLE solution of the marginal distribution in closed form.
\subsection{Probabilistic principal component analysis}\label{sec:PPCA}
Probabilistic PCA (PPCA) introduced independently by \cite{tipping1999probabilistic} and  \cite{roweis1997algorithms} is a probabilistic formulation of classical PCA  which assumes a linear relationship between the observed variable $\mathbf{x}$ and the $K$-dimensional latent variable $\mathbf{z}$ plus a Gaussian noise term $\epsilon \sim \mathcal{N}(0, \sigma^2 \mathbf{I})$ 
\begin{equation}\label{eq:decodePPCA}
	\mathbf{x} = \mathbf{W}\mathbf{z} + \boldsymbol{\mu} + \epsilon
\end{equation}
where $\mathbf{W}$ is a matrix of size ($D\times K$) and $\boldsymbol{\mu}$ is a $D$-dimensional vector mean. 

Assuming a zero mean and uncorrelated Gaussian latent variable $\mathbf{z} \sim \mathcal{N}(0,  \mathbf{I})$, the value of $\mathbf{x}$ conditioned on $\mathbf{z}$ is given by
\begin{equation}\label{eq:xz_ppca}
	p(\mathbf{x}|\mathbf{z}) = \mathcal{N}(\mathbf{x}|\mathbf{W}\mathbf{z} + \boldsymbol{\mu}, \sigma^2 \mathbf{I})
\end{equation}
where the conditional mean equals to 
\begin{equation}\label{eq:xz_ppca_conmean}
    \mathbb{E}[\mathbf{x}|\mathbf{z}] = \mathbf{W}\mathbf{z} + \boldsymbol{\mu}
\end{equation}
and the conditional isotropic covariance ($D\times D$) matrix given by $\sigma^2 \mathbf{I}$. 

The PPCA model is described by a set of parameters $ \boldsymbol{\Theta} = \{\mathbf{W}, \sigma^2 \mathbf{I}\}$ for a total of $(D\times K + 1)$ parameters. 
We can compute the marginal distribution $p(\mathbf{x})$ solving the following integral
\begin{equation}\label{eq:marginal_ppca}
	p(\mathbf{x}) = \int_\mathbf{z} p(\mathbf{x}|\mathbf{z}) p(\mathbf{z}) \,d\mathbf{z}
\end{equation}
that can be computed in closed form since it represents again a convolution of two Gaussian distributions with $p(\mathbf{x}) = \mathcal{N}(\mathbf{x}| \boldsymbol{\mu}, \mathbf{C})$ and a covariance matrix given by
\begin{equation}\label{eq:cov_ppca}
	\mathbf{C} = \mathbf{W}\mathbf{W}^{\top} + \sigma^2\mathbf{I}
\end{equation} 
The marginal distribution depends on the parameters $\boldsymbol{\Theta}$ that can be determined by maximizing the marginal likelihood
\begin{equation}
	\begin{split}\label{eq:max_like_ppca}
		\ln p(\mathbf{X}|\boldsymbol{\Theta} ) 
		&= \sum_{n=1}^{N}\ln p(\mathbf{x}_n|\mathbf{W}, \boldsymbol{\mu}, \sigma^2 ) \\
		&= - \frac{ND}{2} - \frac{N}{2}\ln(2\pi) - \frac{N}{2}\sum_{n = 1}^{N}(\mathbf{x}_n - \boldsymbol{\mu})^{\top}\mathbf{C}^{-1}(\mathbf{x}_n - \boldsymbol{\mu})
	\end{split}
\end{equation}
Setting the derivatives to zero, the likelihood is maximized at
\begin{gather}\label{mar_like}
	\boldsymbol{\mu} = \frac{1}{N}\sum_{i=1}^{N} \mathbf{x}_i \\
	\sigma^2 = \frac{1}{D-K} \sum_{i=K+1}^{D}\lambda_i \\
	\mathbf{W} = \mathbf{U} (\boldsymbol{\Lambda} - \sigma^2 \mathbf{I})^{1/2}\mathbf{R} \label{eq:wmle}
\end{gather}
The marginal likelihood in Eq. \ref{eq:max_like_ppca} is a nonconvex function but all stable stationary points are global minima \cite{tipping1999probabilistic}.

The maximum likelihood solution is obtained in closed form, where $\boldsymbol{\mu}$ is the sample mean, $\boldsymbol{\Lambda}$ is a ($K\times K$) diagonal matrix composed by the top-$K$ largest eigenvalues of the sample covariance matrix $\mathbf{S}$ and $\mathbf{U}$ is a matrix composed by the PCA eigenvectors. 
The expected value of the residual variance of the discarded $N-K$ principal components is represented by $\sigma^2$.
The orthogonal $(K\times K)$ matrix $\mathbf{R}$ can be set to $\mathbf{I}$. In this case, the matrix $\mathbf{W}$ is composed of the principal components $\mathbf{U}$ scaled by a factor of $\sqrt{\lambda_i - \sigma^2}$. 
However, if a numerical optimization algorithm is used to find the maximum likelihood solution, $\mathbf{R}$ can be arbitrary and the resulting matrix $\mathbf{W}$ may not be orthogonal at the optimal solution.

To define a projection of a point $\mathbf{x}$ in the latent space we can compute the posterior distribution of $p(\mathbf{z}|\mathbf{x})$ using the Bayes theorem as 
\begin{equation}\label{eq:zx_ppca}
	p(\mathbf{z}|\mathbf{x}) = \mathcal{N}(\mathbf{z}|\mathbf{M}^{-1}\mathbf{W}^{\top}(\mathbf{x}-\boldsymbol{\mu}), \sigma^2 \mathbf{M}^{-1})
\end{equation}
with $\mathbf{M}= \mathbf{W}^{\top}\mathbf{W} + \sigma^2\mathbf{I}$. The latent representation of a data point $\mathbf{x}$ in the latent space is given by the posterior mean
\begin{equation}\label{eq:encodePPCA}
	\mathbb{E}[\mathbf{z}|\mathbf{x}] = \mathbf{M}^{-1}\mathbf{W}^{\top}(\mathbf{x}-\boldsymbol{\mu})
\end{equation}
and can be reconstructed in data space with 
\begin{equation}\label{eq:decodePPCA2}
\tilde{\mathbf{x}} = \mathbf{W}\mathbb{E}[\mathbf{z}|\mathbf{x}] + \boldsymbol{\mu}    
\end{equation}

Differently from PCA, in PPCA and FA, we are not performing an orthogonal projection of the data in the latent space. In fact, if we take the limit $\sigma^2 \to 0$ with $\mathbf{W}$ given by the MLE solution we recover the classical PCA, 
\begin{equation}
(\mathbf{W}^{\top}\mathbf{W})^{-1}\mathbf{W}^{\top}(\mathbf{x}-\boldsymbol{\mu})    
\end{equation}
but in this case, the posterior covariance is zero and the density becomes singular and then not defined.

Even if an exact closed-form solution of the likelihood is provided, could be advantageous to use an iterative procedure to compute the parameters instead of performing the eigenvalue decomposition of the $(D\times D)$ sample covariance matrix $\mathbf{S}$. This can be an expensive and not viable option in case of very large input dimensionality $D$. 
As for the FA model, the EM algorithm involves in the maximization of the expectation of the complete data log-likelihood $\ln p(\mathbf{X}, \mathbf{Z}|\boldsymbol{\Theta})$ respect to the posterior $p(\mathbf{z}|\mathbf{x})$. For the details, see for example \cite{murphy2012machine, bishop2006-PRM}.

Also for PPCA an efficient computation of the inverse of the marginal covariance $\mathbf{C}$ can be performed as follows 
\begin{equation}\label{eq:C_inv_PPCA}
\mathbf{C}^{-1} = \sigma^{-2} \mathbf{I} - \sigma^{-2} \mathbf{W}\mathbf{M}^{-1}\mathbf{W}^\top
\end{equation}

Finally, PPCA as the FA is invariant to rotations in latent space which means that there exists a set of matrices $\mathbf{W}$ giving the same marginal distribution as shown in Eq. \ref{eq:fa_R}.
\subsection{An Anomaly-Aware Optimization Model for SBDO}\label{sec:mahopt}
Once we trained our model (either a PPCA or FA) on the data matrix $\mathbf{X}$ we can start the SBDO process. The optimization problem is given by
\begin{alignat}{3}
	\min_{\mathbf{z}}  & \qquad           f(\mathbf{x}(\mathbf{z})) \label{eq:rpopt_sbdo1} \\ 
	\text{subject to: }&  \qquad g_j(\mathbf{x}(\mathbf{z}))  \leq a_j \label{eq:rpopt_sbdo2}  &\quad  j=1,\dots,J  \\
	&  \qquad \varphi(\mathbf{x}(\mathbf{z}))  \leq \varphi_{\max} \label{eq:rpopt_sbdo3}  \\
	&  \qquad z_{k}^{\text{lb}} \leq  z_{k} \leq  z_{k}^{\text{ub}} \label{eq:rpopt_sbdo4}&\quad  k = 1,\dots, K    
\end{alignat}
where we added a new constraint in Eq. \ref{eq:rpopt_sbdo3}. 
The term $\varphi(\mathbf{x}(\mathbf{z}))$  represents the anomaly score associated to the geometry $\mathbf{x}(\mathbf{z})$.  
If the anomaly score is high, might indicate potential anomalous and out-of-distribution instances.

Given the inverse of the covariance matrix $\mathbf{C}^{-1}$, the scalar $\varphi(\mathbf{x}(\mathbf{z}))$ can be represented by the value of the marginal distribution estimated through the probabilistic latent variable models
\begin{equation}\label{gaussC}
	\mathcal{N}(\mathbf{x}|\boldsymbol{\mu}, \mathbf{C}) = \frac{1}{2\pi^{D/2} |\mathbf{C}^{1/2}|} \exp \biggl\{-\frac{1}{2} (\mathbf{x} - \boldsymbol{\mu})^{\top}\mathbf{C}^{-1}(\mathbf{x} - \boldsymbol{\mu})\biggr\}
\end{equation}
In this work, we use a more interpretable metric which is given by the exponent of Eq. \ref{gaussC}, namely the squared Mahalanobis distance  \cite{mahalanobis1936generalized} $d^2_M(\mathbf{x})$ defined as 
\begin{equation}\label{eq:mah_dist}
	d^{2}_M(\mathbf{x}) =(\mathbf{x} - \boldsymbol{\mu})^{\top} \mathbf{C}^{-1}(\mathbf{x} - \boldsymbol{\mu})
\end{equation}
representing an extension of the Euclidean distance which takes into account the correlation between vectors. For $\mathbf{C} = \mathbf{I}$, Eq. \ref{eq:mah_dist} reduces to the Euclidean distance.

The scalar $\varphi_{\max}$ represents the upper limit on the allowed anomaly score for the shape generated by the optimizer within the latent space. 
We calculate $\varphi_{\max}$ with respect to the reconstructed dataset $\tilde{\mathbf{X}}$, by encoding each data point $\mathbf{x}$ in the latent space using Eq. \ref{eq:encodeFA} or Eq. \ref{eq:encodePPCA} and then projecting it back in data space using Eq. \ref{eq:decodeFA2} or Eq. \ref{eq:decodePPCA2}, respectively for FA and PPCA. We then evaluate Eq. \ref{eq:mah_dist} for every $\tilde{\mathbf{x}}\in\tilde{\mathbf{X}}$. The value of $\varphi_{\max}$ is determined in this work using the interquartile range (IQR) rule.
\begin{figure}[htb!]
	\centering
	\includegraphics[width=0.8\textwidth]{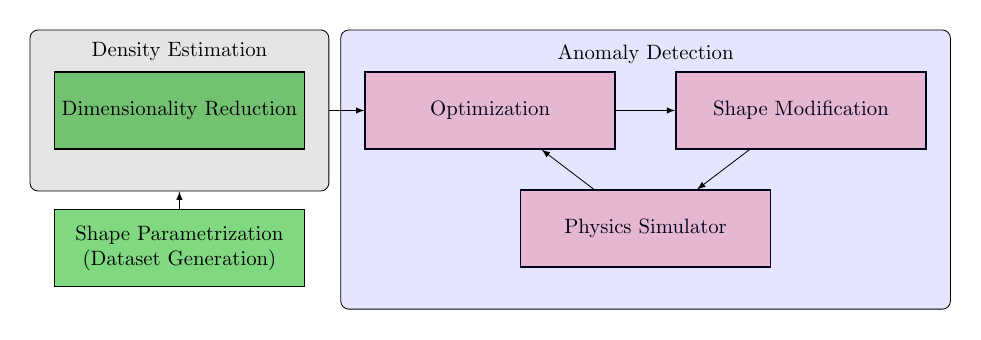}
	\caption{Graphical representation of the new SBDO framework (DR-DE-SBDO) for anomaly detection.}\label{fig:sbdo_druq}
\end{figure}
The DR-DE-SBDO for anomaly detection procedure is summarized in Procedure \ref{proc:druqsbdo} and the corresponding scheme in Fig. \ref{fig:sbdo_druq}.

It would be worthwhile to investigate the general distribution of the random variable $d^2_M(\mathbf{x})$. Now we show that under some conditions the squared Mahalanobis distance follows a chi-square distribution $d^2_M(\mathbf{x}) \sim \chi^{2}_{D}$.
\begin{lemma}\label{lemma:chi}
	Given a $D$-dimensional random variable vector $\mathbf{x}$, then if $\mathbf{x}\sim \mathcal{N}(\mathbf{x}|\boldsymbol{\mu}, \boldsymbol{\Sigma})$ then the distribution of the squared Mahalanobis distance follows a chi-square distribution $d^2_M(\mathbf{x}) \sim \chi^{2}_{d}$. 
\end{lemma}
\begin{proof}
	For the random variable $\mathbf{x}$ the Gaussian distribution is given by
	\begin{equation}
		\mathcal{N}(\mathbf{x}|\boldsymbol{\mu}, \mathbf{\Sigma}) = \frac{1}{2\pi^{D/2} |\boldsymbol{\Sigma}^{1/2}|} \exp \biggl\{-\frac{1}{2} (\mathbf{x} - \boldsymbol{\mu})^{\top}\boldsymbol{\Sigma}^{-1}(\mathbf{x} - \boldsymbol{\mu})\biggr\}
	\end{equation}
	consequently, the functional dependence with respect to $\mathbf{x}$ is given only with respect to the squared Mahalanobis distance $d^2_M(\mathbf{x})$ inside the exponent of the Gaussian distribution. Using the eigenvalue decomposition for the covariance matrix 
	\begin{equation}
	    \boldsymbol{\Sigma}\mathbf{u}_i = \mathbf{u}_i \lambda_i \qquad i = 1,\dots, D
	\end{equation}
	its inverse can be written as follows
	\begin{equation}
		\boldsymbol{\Sigma}^{-1} = \sum_{i = 1}^{D} \frac{1}{\lambda}_i \mathbf{u}_i\mathbf{u}_i^{\top}
	\end{equation}
	and the squared Mahalanobis distance becomes 
	\begin{equation}
		d^2_M(\mathbf{x}) = \sum_{i = 1}^{D} \frac{1}{\lambda}_i(\mathbf{x} - \boldsymbol{\mu})^{\top} \mathbf{u}_i\mathbf{u}_i^{\top}(\mathbf{x} - \boldsymbol{\mu}) =  \sum_{i = 1}^{D} \frac{y_i^{2}}{\lambda_i}
	\end{equation}
	where $y_i = \mathbf{u}_i^{\top}(\mathbf{x} - \boldsymbol{\mu})$ represents the $i$-th component of the vector $\mathbf{y}$ in the new coordinate system defined by the eigenvectors $\mathbf{U}$. The random variable $\mathbf{x}$ is Gaussian and the affine transformation $\mathbf{y}$ is again Gaussian, consequently, since for definition the sum of the squared value of normal random variables follows a chi-square distribution then $d^2_M(\mathbf{x})\sim \chi^{2}_{d}$.
\end{proof}
This result can be used as an additional assessment of the nature of the distribution of the random variable $\mathbf{x}$ and it will be utilized in section \ref{subsec:mah}. 
\begin{procedure}[htb!]
  \caption{DR-DE-SBDO().}
  \label{proc:druqsbdo}
  Generate a dataset $\mathbf{X}$ by sampling the design variable vector $\mathbf{v}$ from a uniform distribution, as outlined in section \ref{sec:genFFD} \;
  Train PPCA/FA with the data $\mathbf{X}$ to obtain the set of parameters $\boldsymbol{\Theta}$ \; 
  Compute the inverse of the covariance matrix $\mathbf{C}^{-1}$ using Eq. \ref{eq:C_inv_PPCA} for PPCA or Eq. \ref{eq:C_inv_FA} for FA \; 
  Fix the latent space dimensionality $K$ to the desired level of explained geometrical variance \;
  Determine the scalar $\varphi_{\max}$ introduced in Eq. \ref{eq:rpopt_sbdo3} as follows \newline
  \For{$n$ \KwTo $N$ }{Compute $\mathbb{E}[\mathbf{z}_n|\mathbf{x}_n]$ from Eq. \ref{eq:encodeFA} or Eq. \ref{eq:encodePPCA} for FA and PPCA respectively \; 
  Compute the reconstructed data $\tilde{\mathbf{x}}_n = \mathbf{W}\mathbb{E}[\mathbf{z}_n|\mathbf{x}_n] + \boldsymbol{\mu}$ \;
  Compute the squared Mahalanobis distance $d^{2}_{M}(\tilde{\mathbf{x}}_n) = (\tilde{\mathbf{x}}_n - \boldsymbol{\mu})^{\top} \mathbf{C}^{-1}(\tilde{\mathbf{x}}_n - \boldsymbol{\mu})$ and collect it inside the vector $\mathbf{d}^2_{M}(\tilde{\mathbf{x}})$\;} 
  Set the value $\varphi_{\max}$ to the 3/2 the interquartile range of the values collected in $\mathbf{d}^2_{M}(\tilde{\mathbf{x}})$ \;
  Start the SBDO process with respect to the latent variable $\mathbf{z}\in \mathbb{R}^K$ and fix the maximum number of function evaluations $I_{max}$ (or any other stopping criteria for the optimization)\; 
  \For{$i$ \KwTo $I_{max}$ }{For a given $\mathbf{z}_i$ from the optimization algorithm project it back in the input space $\mathbf{x}_i = \mathbf{W}\mathbf{z}_i + \boldsymbol{\mu}$ \;
  Evaluate the objective function $f(\mathbf{x}_i)$ \; }
\end{procedure}

To conclude this section, we analyze the effect of the constraint in Eq. \ref{eq:rpopt_sbdo3} on the SBDO process from a geometrical perspective. Specifically, the constraints in Eq. \ref{eq:rpopt_sbdo3} and Eq.  \ref{eq:rpopt_sbdo4} define a hyperellipsoid and a hyperrectangle in the data space $\mathbb{R}^{D}$, respectively. To simplify the analysis, we can approximate the hyperellipsoid with a hypersphere, corresponding to an isotropic covariance matrix for the Gaussian distribution. Similarly, we assume a hypercube instead of a hyperrectangle to provide a more general analysis.
We can derive an expression for the volume of the hypersphere and for the hypercube, using simple tools from the geometry of $\mathbb{R}^{D}$ \cite{Kendall1962ACI}. 
The volume of the hypersphere of radius $r$ is given by 
\begin{equation}
    V_D = \pi ^{D/2}r^{D} (\Gamma(D/2 + 1))^{-1}
\end{equation}
with the Gamma function in this case defined as $\Gamma(D) = \int_{0}^{\infty} t^{D -1} e^{-t} dt$, with the well known property that $\Gamma(D) = (D-1)!$. 
The volume of the hypercube of length $l =2r$ is given by 
\begin{equation}
    H(D) = l^D = 2r^D
\end{equation}
The asymptotic behavior for the ratio of the volume of the hypersphere inscribed in a hypercube is given by
\begin{equation}\label{eq:limvol}
	\lim_{D \to +\infty} \frac{V_D}{H_D} = \frac{\pi ^{D/2}r^{D}}{2r^D \Gamma(\frac{D}{2} + 1)} = 0 
\end{equation}
with
\begin{equation}
	\Gamma\bigg(\frac{D}{2} + 1\bigg) =  
	\begin{cases}
		(\frac{D}{2})! & \text{if}\ \,\, D \,\, \text{is even}  \\
		\sqrt{\pi}(\frac{D!!}{2^{d+1/2}}) & \text{if}\ \,\, D \,\, \text{is odd}
	\end{cases}
\end{equation} 
for $D = 2$ the limit in Eq. \ref{eq:limvol} is equal to $78.5\%$, meaning that the circle covers the $78.5\%$ of the space defined by the hypercube. The limit converges very fast to zero. For $D = 6$, the hypersphere covers only the $\simeq 8\%$ of the entire hypercube. 
This shows that as the dimensionality increase, the majority of the volume of the hypercube is concentrated in its $2^D$ corners.

A graphical representation is given in Fig. \ref{fig:cube_sphere} showing that the radius of the inscribed circle accurately measures the discrepancy in volume between the hypercube and the inscribed hypersphere in $D$ dimensions. For our application, the white areas are where anomalous geometries might belong while the grey areas represent the region where geometries in the original parametrization are confined. 
\begin{figure}[htb!]
	\centering
	\subfigure[$D=2$]{\label{fig:cube_sphere_d2}\includegraphics[width=0.18\textwidth]{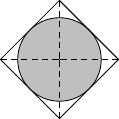}}
	\subfigure[$D=3$]{\label{fig:cube_sphere_d3}\includegraphics[width=0.18\textwidth]{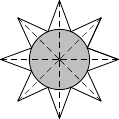}}
	\subfigure[$D=4$]{\label{fig:cube_sphere_d4}\includegraphics[width=0.18\textwidth]{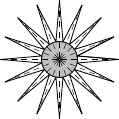}}
	\subfigure[$D=5$]{\label{fig:cube_sphere_d5}\includegraphics[width=0.18\textwidth]{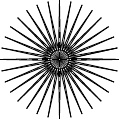}}
	\subfigure[$D=6$]{\label{fig:cube_sphere_d6}\includegraphics[width=0.18\textwidth]{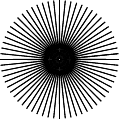}}
\caption{Conceptualization about the behavior of a sphere inscribed in a hypercube in high-dimensional spaces.}
\label{fig:cube_sphere}
\end{figure}
This indicates that not considering the constraint in Eq. \ref{eq:rpopt_sbdo3} in the optimization model could allow the optimizer
to search into an arbitrarily large design space, with a high probability to produce anomalous geometries that are very far from the center of the distribution (i.e. the mean). 
\section{Application: Shape Optimization of a Naval Destroyer DTMB 5415}\label{sec:7}
In this section we apply the new proposed framework for the optimization of a US naval Destroyer, the DTMB 5415 model (Fig. \ref{fig:dtmb_model}) widely investigated by towing tank experiments as in \cite{olivieri2001towing}.
\begin{figure}[htb!]
	\centering
	\includegraphics[width=10cm]{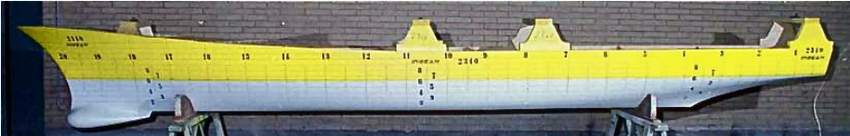}
	\caption{A 5.720 m length model of the DTMB 5415 (CNR-INM model 2340).}
	\label{fig:dtmb_model}
\end{figure}
\begin{table}[htb!]
	\begin{center}
		\caption{DTMB 5415 model scale main particulars and test condition.}
		\label{tab:main_particulars}
		\small
		\begin{tabular}{lcc}
			\toprule
			Description          & Unit & Value \\
			\midrule
			Displacement                  & tonnes   & 0.549\\
			Length between perpendiculars & m        & 5.720\\
			Beam                          & m        & 0.760\\
			Draft                         & m        & 0.248\\
			Longitudinal center of gravity& m        & 2.884\\
			Vertical center of gravity    & m        & 0.056\\
			Water density                 & kg/m$^3$ & 998.5 \\
			Kinematic viscosity           & m$^2$/s  & 1.09E-06\\
			Gravity acceleration          & m/s$^2$  & 9.803 \\
			Froude number                 & --       & 0.280\\
			\bottomrule
		\end{tabular}
	\end{center}
\end{table}
\subsection{Design-Space Parametrization and Sampling}
The design variables ($M=21$) are sampled from a uniform random distribution obtaining $N = 10000$ hull-form feasible design from the FFD procedure.  Designs not satisfying the geometrical constraints are not included in the data matrix so all designs processed successively by the dimensionality reduction models are feasible. 
\begin{table*}[htb!]
	\centering
	\small
	\caption{Hull shape modification, FFD control points, and variables setup.}\label{tab:FFD_hull}
	\begin{tabular}{clccc}
		\toprule
		Layer & Layer $\mathsf{x}$-plane & No. CPs & No. active CPs & Variable range \\	
		\midrule	
		1 & $\mathsf{x}=0.00$ & 12 & 1   & $-1.0\leq v_y^{(1,2)}\leq 1.0$  \\
		2 &$\mathsf{x}=18.21$ & 12 & 2    & $-1.0\leq v_y^{(3,4)}\leq 1.0$  \\
		3 &$\mathsf{x}=36.42$ & 12 & 2      & $-1.0\leq v_y^{(5,6)}\leq 1.0$ \\
		4 &$\mathsf{x}=54.63$ & 12 & 2      & $-1.0\leq v_y^{(7,8)}\leq 1.0$ \\
		5 &$\mathsf{x}=72.85$ & 12 & 2      & $-1.0\leq v_y^{(9,10)}\leq 1.0$ \\
		6 &$\mathsf{x}=91.06$ & 12 & 2       & $-1.0\leq v_y^{(11,12)}\leq 1.0$ \\
		7 &$\mathsf{x}=109.27$ & 12 & 1 & $-1.0\leq v_y^{13}\leq 1.0$  \\
		7 &$\mathsf{x}=109.27$ & 12 & 1   & $-2.0\leq v_y^{14}\leq 2.0$  \\
		8 &$\mathsf{x}=127.49$ & 12 & 1  & $-2.0\leq v_y^{15}\leq 2.0$  \\
		9 &$\mathsf{x}=145.70$ & 12 & 1   & $-2.0\leq v_y^{16}\leq 2.0$ \\                         
		\bottomrule
	\end{tabular}
\end{table*}
\begin{table*}[htb!]
	\centering
	\small
	\caption{Bulb shape modification, FFD control points, and variables setup.}\label{tab:FFD_bulb}
	\begin{tabular}{clccc}
		\toprule
		Layer & Layer $\mathsf{x}$-plane & No. CPs & No. active CPs &  Variable range \\	
		\midrule	
		1 & $\mathsf{x}=127.20$ & 9 & 0   & (-)  \\
		2, 3 & $\mathsf{x}=133.20 \land \mathsf{x}=139.20$ & 18 & 1  & $-1.0\leq v_y^{17}\leq 1.0$  \\
		2, 3 & $\mathsf{x}=133.20 \land \mathsf{x}=139.20$ & 18 & 1  & $-1.0\leq v_z^{18}\leq 1.0$  \\
		2, 3 &$\mathsf{x}=133.20 \land \mathsf{x}=139.20$ & 18 & 1  & $-0.3\leq v_x^{19}\leq 0.3$ \\
		4 & $\mathsf{x}=139.20$ & 9 & 1 & $-1.0\leq v_z^{20}\leq 1.0$  \\
		4 &$\mathsf{x}=139.20$ & 9 & 1  & $0.0\leq v_x^{21}\leq 0.5$ \\                        
		\bottomrule
	\end{tabular}
\end{table*}
\begin{figure}[htb!]
	\centering
	\subfigure[FFD control points over the hull.]{\label{fig:ffd_hull}\includegraphics[width=0.49\textwidth]{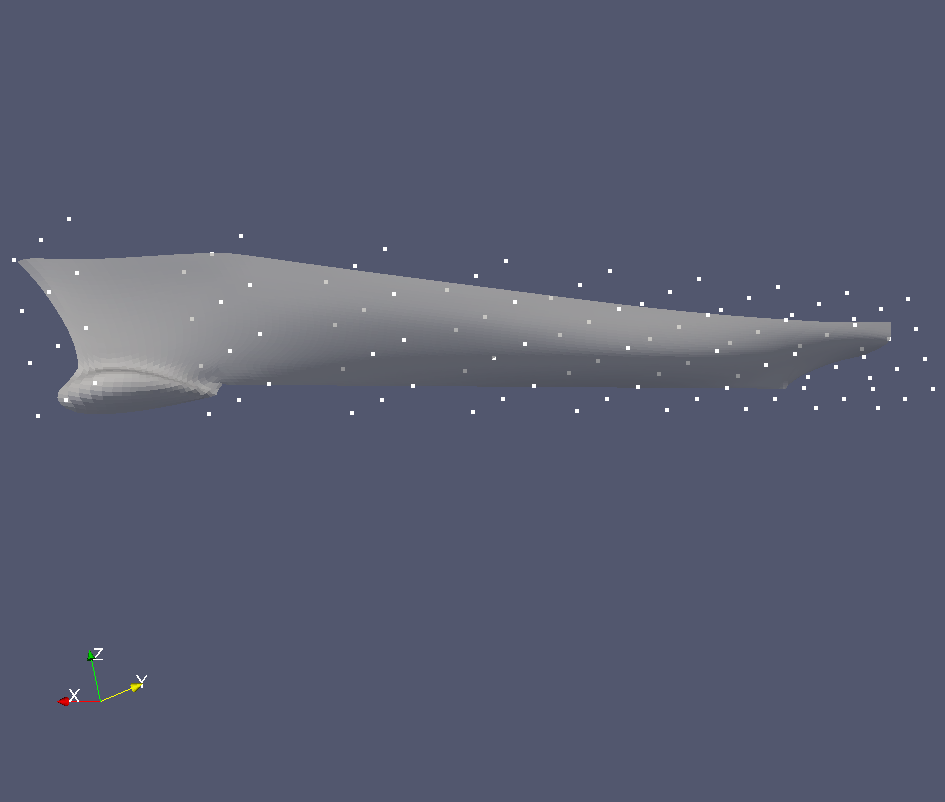}}
\subfigure[FFD control points over the bulb.]{\label{fig:ffd_bulb}\includegraphics[width=0.49\textwidth]{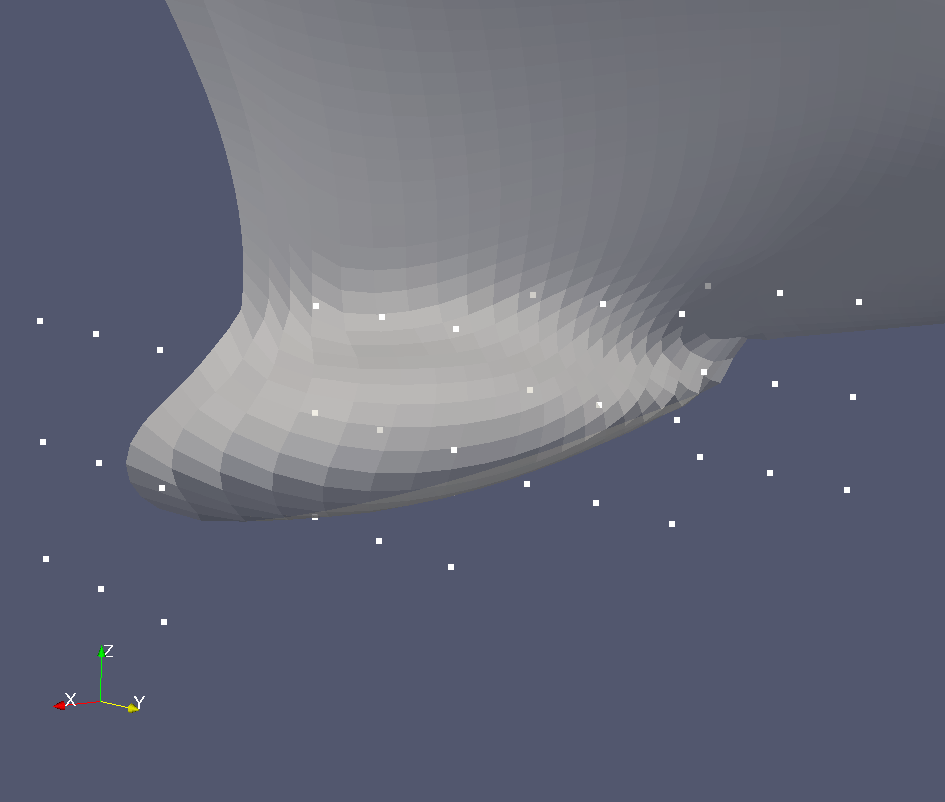}}
\caption{FFD control points.}
\label{fig:FFD}
\end{figure}

The data matrix $\mathbf{X}$ collects a $L=7200$ grids points from hull discretization.  
The resulting dimensionality in the data space $D$ is equal to $D=L\times 3 = 21600$. 
We used two different and independent FFD lattices for the hull and the bulb as shown in Fig. \ref{fig:ffd_hull} and in Fig. \ref{fig:ffd_bulb} respectively. The number of the design variables controlling the shape modification along the hull is $15$ and $6$ design variables control the shape modification at the bulb level. The design modifications for this application are allowed in all three Cartesian components for the bulb and only in the $\mathsf{y}$ direction over the hull. 
A detailed definition of the design space is given in Tab. \ref{tab:FFD_bulb} and in Tab. \ref{tab:FFD_hull}.
For the FFD implementation we used PyGeM \cite{TezzeleDemoMolaRozza2020PyGeM}.
\subsection{Dimensionality Reduction Process}
The first step is to find the right dimensionality of the latent space $K$. 
This can be easily assessed by computing the explained geometrical variance resolved by the dimensionality reduction models varying the number of components $K$. 
For this application, we fix a threshold of $99 \%$ for the explained variance by the PPCA and FA.
\begin{figure}[htb!]
	\centering
	\includegraphics[width=0.49\textwidth]{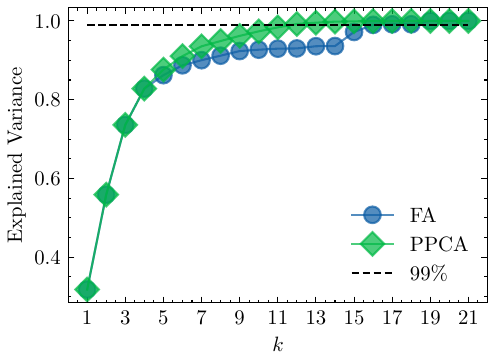}
	\caption{Design-space variability retained as a function of the number $K$-components.}
	\label{fig:exvar}
\end{figure}
\begin{figure}[htb!]
	\centering
	\subfigure[PPCA]{\label{fig:eig1_ppca}\includegraphics[width=0.45\textwidth]{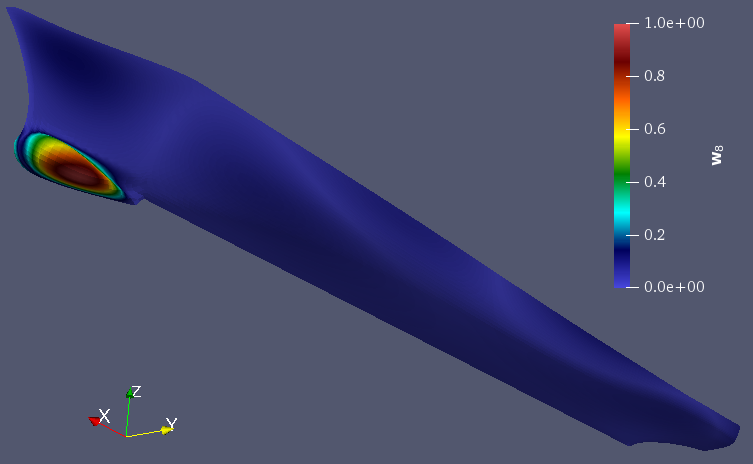}}
	\subfigure[FA]{\label{fig:eig1_fa}\includegraphics[width=0.45\textwidth]{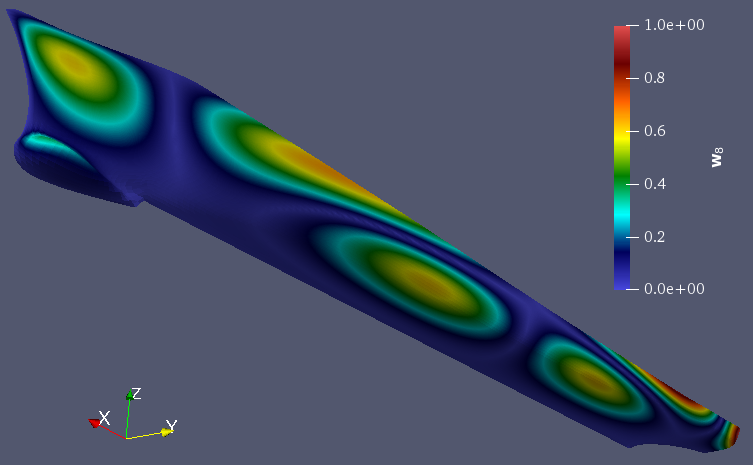}}
	\caption{Graphical representation of PPCA and FA components.}
	\label{fig:eig1}
\end{figure}

In Fig. \ref{fig:exvar} there is the convergence of the explained variance for the dimensionality reduction models. 
For the PPCA the components are computed using the closed-form solution of the marginal likelihood in Eq. \ref{eq:wmle}. In this case, the explained variance is given by the PCA eigenvalues since $\mathbf{W}$ and $\mathbf{U}$ span the same subspace.
For the FA model, we used the EM algorithm to find model parameters, and the variance explained by each component is given by $||\mathbf{w}_k||^{2}$. 
From Fig. \ref{fig:exvar} we observe that PPCA reaches the threshold of $99\%$ with $K=12$ components while the FA required $K=16$. 
In this case, PPCA components are essentially scaled PCA eigenvectors, which results in PPCA being superior to FA in terms of variance explained. 
This is because FA does not directly maximize variance in its objective function (as PCA does), as seen in Eq. \ref{eq:max_like_fa}.

As stated before the columns of the matrix $\mathbf{W}$ represent spatial geometrical components. To illustrate this statement graphically, Fig. \ref{fig:eig1} displays the eighth component $\mathbf{w}_8$ of PPCA and FA. 
\subsection{Fixing the Threshold for the Mahalanobis Distance}\label{subsec:mah}
\begin{wraptable}{r}{5.5cm}
	\begin{center}
		\small
		\caption{Dimensionality reduction and density estimation results.}
		\label{tab:dim_mah}
		\small
		\begin{tabular}{lcc}
			\toprule
			Method         & $K$ & $\varphi_{\max}$ \\
			\midrule
			PPCA            & 12   & 21.74\\
			FA 			    & 16   & 27.13\\
			\bottomrule
		\end{tabular}
	\end{center}
\end{wraptable}
Before starting the SBDO process we need to fix the constant term $\varphi_{\max}$ inside the constraint in Eq. \ref{eq:rpopt_sbdo3}. 
The initial step involves computing the reconstructed data $\tilde{\mathbf{\mathbf{X}}}$. 
For each $\mathbf{x}_n \in \mathbf{X}$, the conditional latent mean is computed for PPCA and FA using Eq. \ref{eq:encodePPCA} and Eq. \ref{eq:encodeFA}, respectively. Next, all the $\mathbf{z}_n$ are projected back by computing the conditional mean for the PPCA and FA models using Eq. \ref{eq:decodePPCA2} and Eq. \ref{eq:decodeFA2}, respectively.
Once we compute the matrix $\tilde{\mathbf{\mathbf{X}}}$ we calculate the squared Mahalanobis distance in Eq. \ref{eq:mah_dist} for each reconstructed geometry $\tilde{\mathbf{x}}\in\tilde{\mathbf{X}}$  . 
At the end of this process, we have a vector of size $N$, namely $\mathbf{d}^2_{M}(\tilde{\mathbf{x}})$.

In this application, the threshold $\varphi_{\max}$ is fixed with respect to 1.5 times the IQR of the values collected inside the vector $\mathbf{d}^2_{M}(\tilde{\mathbf{x}})$
\begin{equation}
    \text{IQR} = q_3 - q_1
\end{equation}
where $q_1$ and $q_3$ are the first and the third quartile.
In Tab. \ref{tab:dim_mah} we summarized the values of the threshold $\varphi_{\max}$ from the two models.

Furthermore, we conducted a simple empirical experiment to observe the impact of the density estimation process. Specifically, we randomly sampled $10,000$ latent variables $\mathbf{z}$ uniformly at random within the bounds defined in Eq. \ref{eq:rpopt_sbdo4}. 
We then projected these samples back into the data space using Eq. \ref{eq:decodePPCA2} and Eq. \ref{eq:decodeFA2} for PPCA and FA, respectively.
We computed the squared Mahalanobis distance for each of those samples denoted by $d^2_M(\mathbb{E}[\mathbf{x}|\mathbf{z}])$ and we compared the resulting probability distribution function  with that of $d^{2}_{M}(\tilde{\mathbf{x}})$, shown by the blue and green lines in Fig. \ref{fig:dist}, respectively. 
We performed this procedure for both PPCA and FA and reported the results in Fig. \ref{fig:dist}.
We can notice that the center of mass of the distribution of $d^2_M(\mathbb{E}[\mathbf{x}|\mathbf{z}])$ is shifted towards a higher value of the squared Mahalanobis distance with respect to the distribution of  $d^{2}_{M}(\tilde{\mathbf{x}})$. 
\begin{figure}[htb!]
	\centering
	\subfigure[PPCA ($K=12$)]{\label{fig:a}\includegraphics[width=8cm]{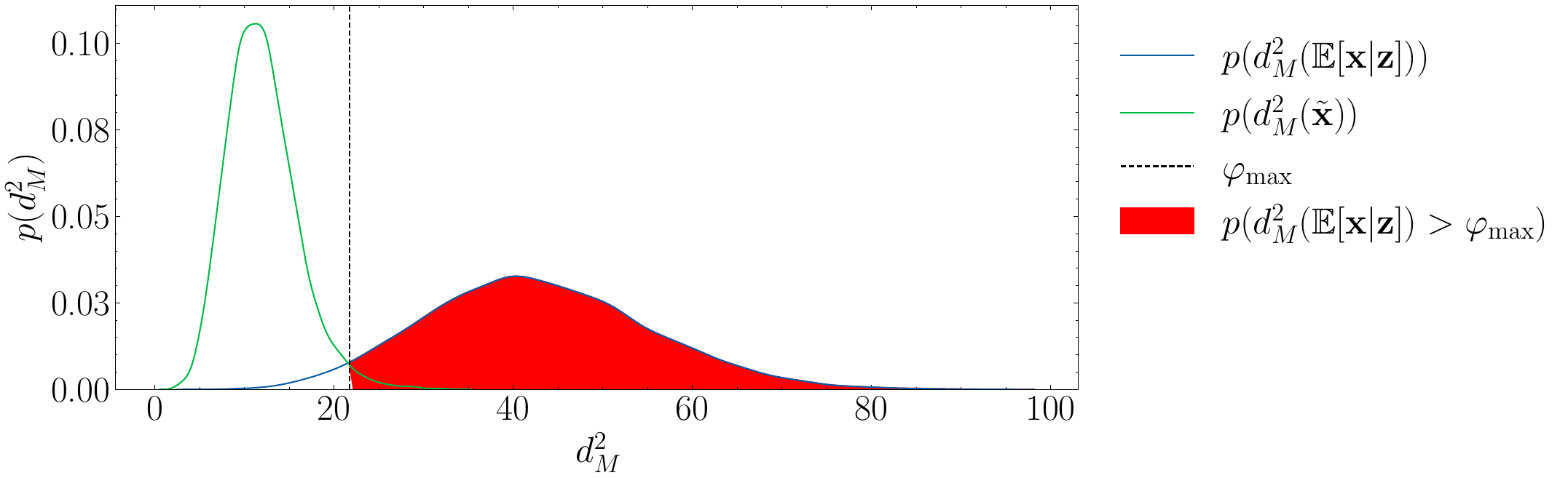}}
	\subfigure[PPCA ($K=12$)]{\label{fig:b}\includegraphics[width=4cm]{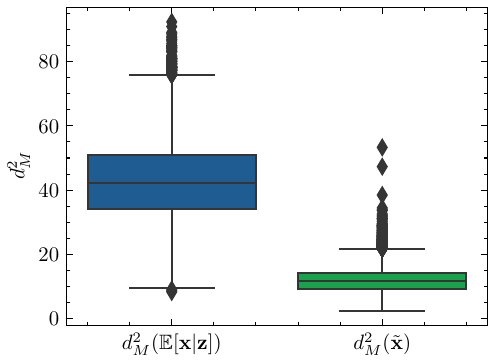}}\\
	\subfigure[FA ($K=16$)]{\label{fig:c}\includegraphics[width=8cm]{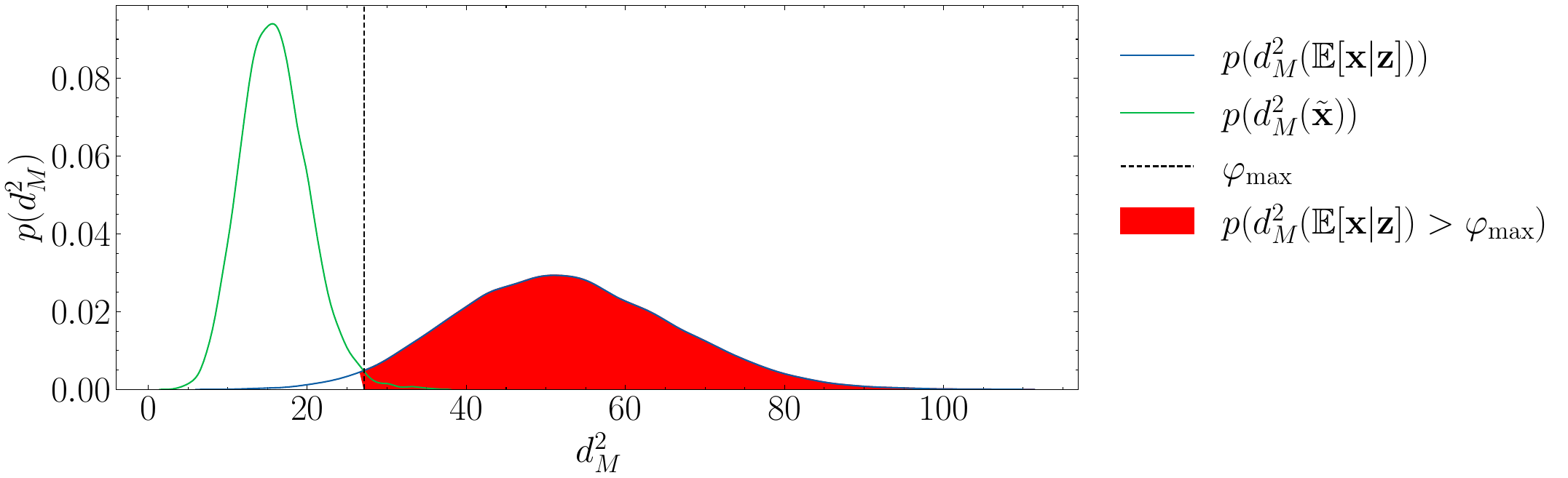}}
	\subfigure[FA ($K=16$)]{\label{fig:d}\includegraphics[width=4cm]{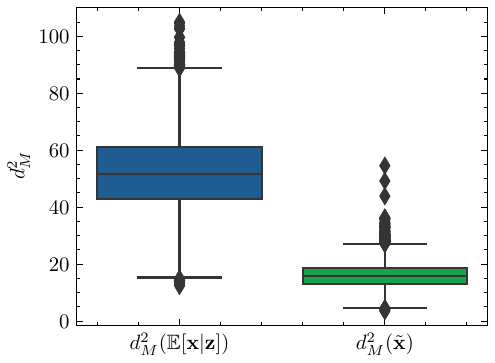}}
	\caption{Behavior in terms of the empirical PDF of the squared Mahalanobis distance (left column) and relative box plots (right column).}
	\label{fig:dist}
\end{figure}
This implies that when a latent variable $\mathbf{z}$ is randomly sampled uniformly within its bounds, it is highly probable that the resulting design will be far from the mean $\mathbf{\bar{x}}$ of the original parametrization.
To quantify this effect, we computed the probability that $d^2_{M}(\mathbb{E}[\mathbf{x}|\mathbf{z}])$ exceeds our threshold $\varphi_{\max}$, denoted by $p(d^2_{M}(\mathbb{E}[\mathbf{x}|\mathbf{z}]) > \varphi_{\max})$ and highlighted in red in Fig. \ref{fig:dist}. 
Our experiment showed that this probability was approximately $0.98$ and $0.99$ for PPCA and FA, respectively. 
Thus, if we perform a random optimization in the latent space using a uniform distribution as an optimizer,  almost all the resulting designs will be far from the mean of the original parametrization.
\begin{figure}[htb!]
	\centering
	\subfigure[$d^2_{M}(\bar{\mathbf{x}}) = 0$]{\label{fig:fro_mean}\includegraphics[width=0.23\textwidth]{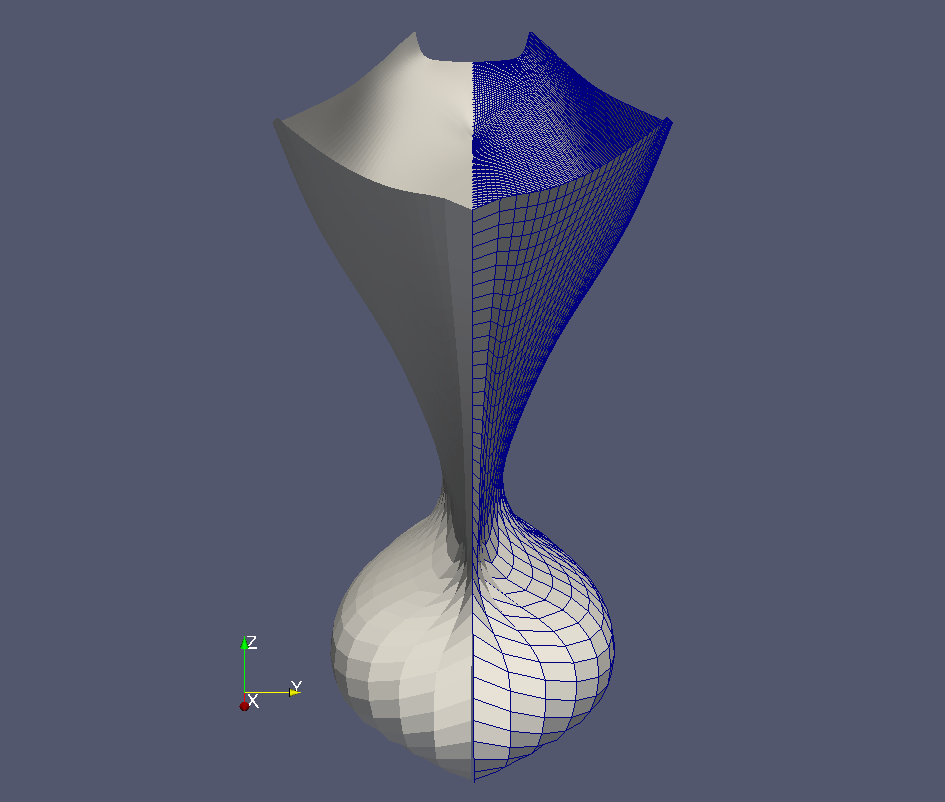}}\\
	\subfigure[$d^2_{M}(\mathbf{x}) = 9.41$]{\label{fig:ano_1}\includegraphics[width=0.23\textwidth]{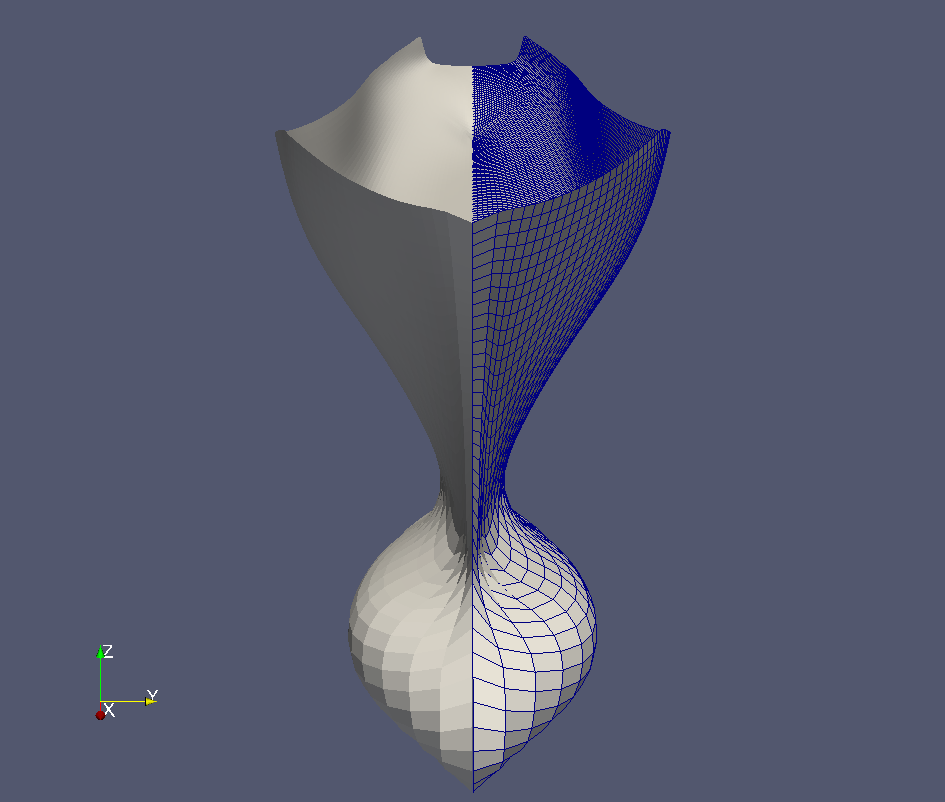}} 
	\subfigure[$d^2_{M}(\mathbf{x}) = 62.41$]{\label{fig:ano_2}\includegraphics[width=0.23\textwidth]{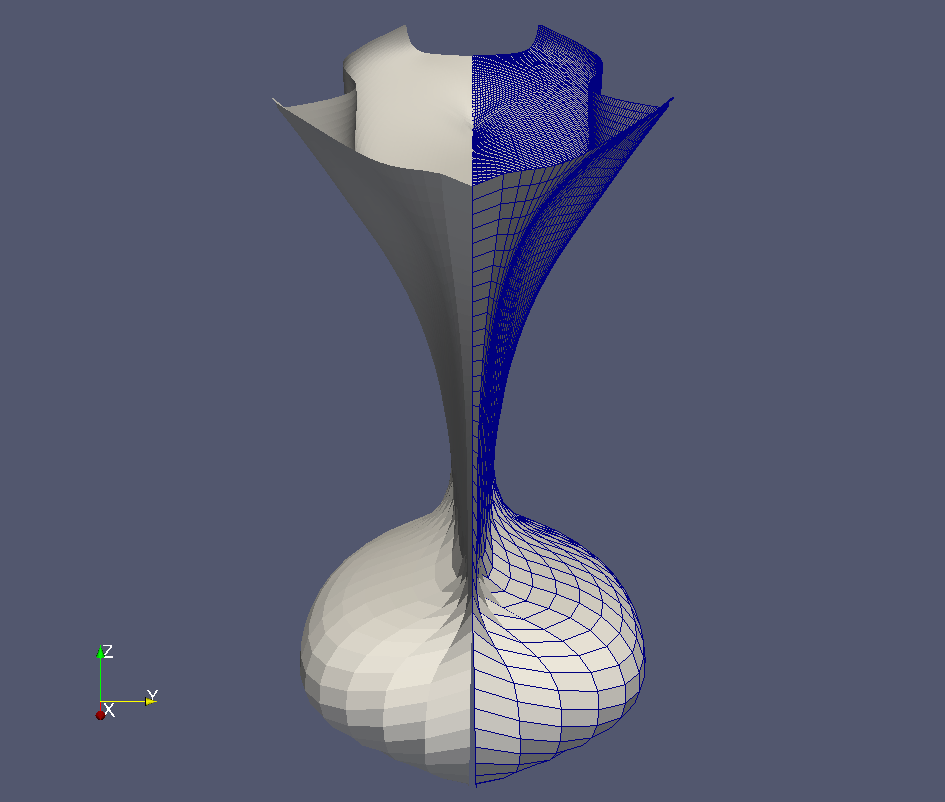}} 
	\subfigure[$d^2_{M}(\mathbf{x}) = 236.15$]{\label{fig:ano_3}\includegraphics[width=0.23\textwidth]{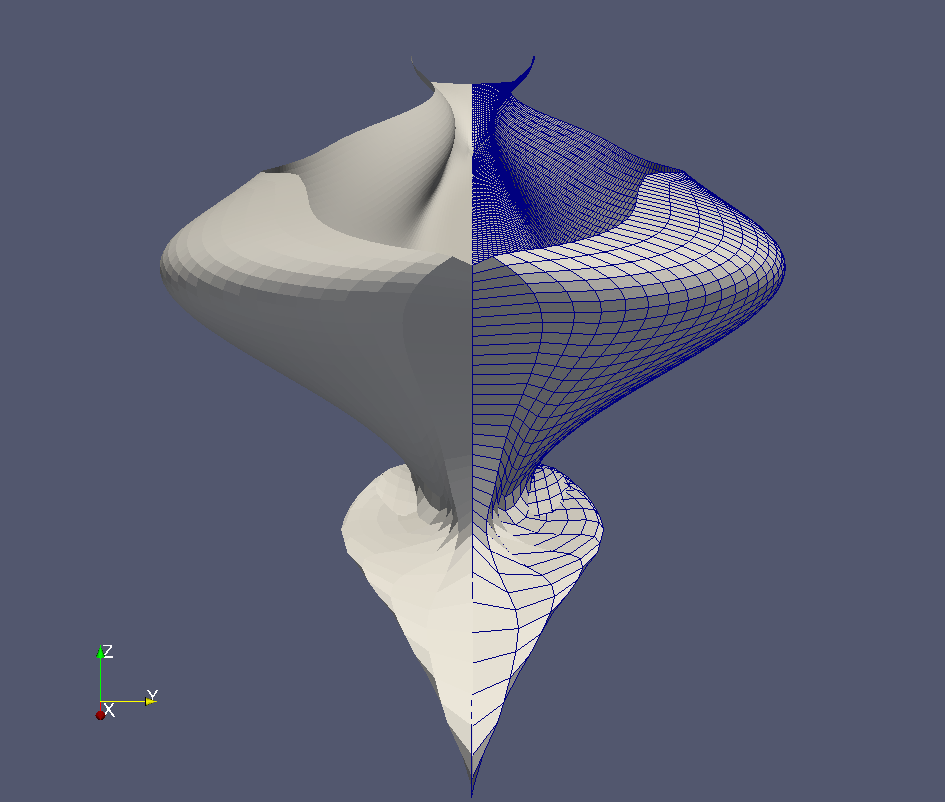}} 
	\subfigure[$d^2_{M}(\mathbf{x}) =688.94$]{\label{fig:ano_4}\includegraphics[width=0.23\textwidth]{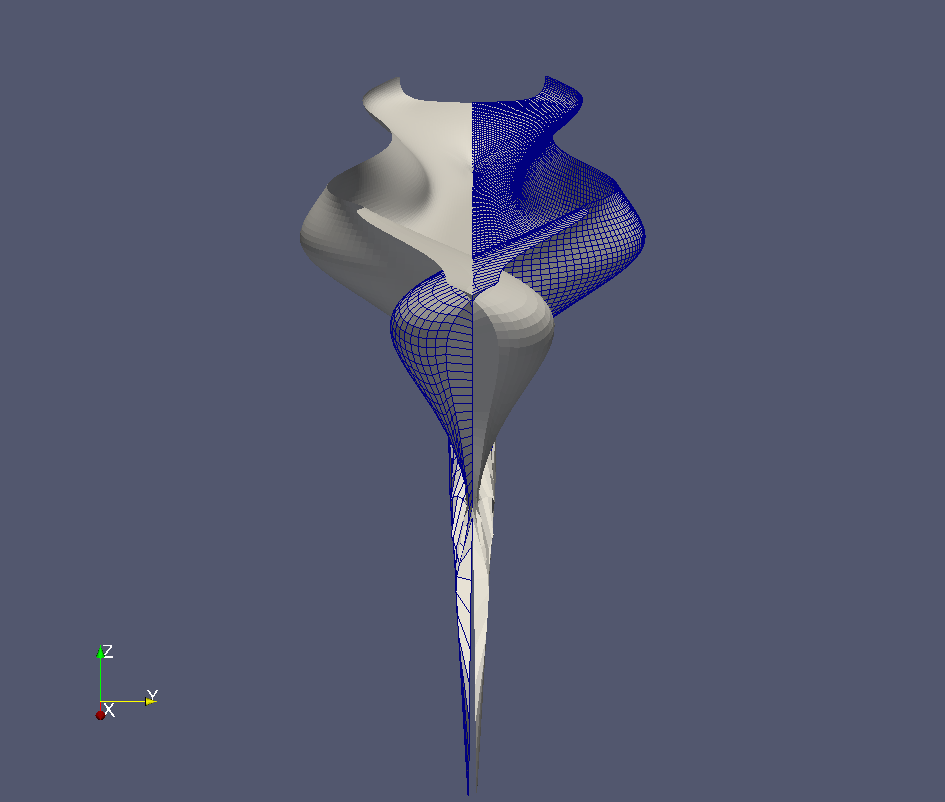}} 
	\caption{Example of geometries $\mathbf{x}$ and relative (squared) Mahalanobis distance from the mean $\bar{\mathbf{x}}$.}
	\label{fig:fa_ppca_anom_unif}
\end{figure}
In Fig. \ref{fig:dist}, we can also observe that the empirical PDF of $p(d^2_{M}(\tilde{\mathbf{x}}))$ resembles a chi-square distribution with $K$ degrees of freedom as additional information that our dataset distribution follows approximately a Gaussian distribution (see Lemma \ref{lemma:chi}).

Finally, to investigate the relationship between the value of the Mahalanobis distance and the regularity of the geometrical modification obtained, we generated again some random geometries sampling the latent variables from a uniform distribution both for FA and PPCA. 
In Fig. \ref{fig:fa_ppca_anom_unif}, is possible to visualize some examples of the designs obtained from this experiment. 
In Fig. \ref{fig:fro_mean}, there is the mean geometry. 
In Fig. \ref{fig:ano_1}, the design has a quite strong similarity with the mean geometry apart in the bulb region where it is slightly thinner. 
In the second geometry in Fig. \ref{fig:ano_2}, it is possible to observe stronger curvatures on the hull. 
This behavior is heavier for the third geometry in Fig. \ref{fig:ano_3} since it gains a very sharp bulb and an excessively large beam on the bow region of the hull. The design in Fig. \ref{fig:ano_4}, is even more extreme as it no longer resembles the geometrical properties of a hull.

In general, stronger and anomalous shape modifications may be obtained as the Mahalanobis distance or the uncertainty of our model about a particular design increases.
\subsection{Optimization Problem}
The problem formulation for the shape optimization of the DTMB 5415 is the following
\begin{equation}\label{eq:5415problem}
	\begin{array}{rll}
		\min\limits_{\mathbf{v}}     & R_\mathrm{T}(\mathbf{x(\mathbf{v})}) & \mathrm{with} \,\,\, \mathbf{v}\in\mathbb{R}^M\\
		& L_{\rm pp}({\mathbf{x(\mathbf{v})}})= L_{\rm pp_0}\\
		& \nabla(\mathbf{x(\mathbf{v})})= \nabla_0,\\
		& |\Delta B(\mathbf{x(\mathbf{v})})|\leq 0.05B_0,\\
		& |\Delta T(\mathbf{x(\mathbf{v})})|\leq 0.05T_0, \\
		& V(\mathbf{x(\mathbf{v})}) \geq V_0,\\
		& v_m^{\text{lb}}\leq v_m \leq v_m^{\text{ub}} &  m=1,\dots,M \\
	\end{array}
\end{equation}
where $R_{\rm T}$ is the calm-water resistance at ${\rm Fr}=0.28$ (equivalent to 20 kn for the full-scale ship). 
Equality constraints are defined for the length between perpendiculars ($L_{\rm pp}$) and the displacement ($\nabla$). 
Inequality constraints include $5\%$ of maximum variation of the beam ($B$) and the drought ($T$) and dedicated volume for the sonar dome ($V$), corresponding to 4.9 m diameter and 1.7 m length (cylinder). Subscript `0' indicates original-geometry values. 
Finally, to satisfy the equality constraints, scaling equations are applied to the modified geometry as in \cite{dagostino2020design}.
Equality and inequality constraints on the geometry deformations are based on \cite{diez2018-AVT204-3}.

Using the reduced-dimensionality design space with probabilistic linear latent variable models PPCA and FA, the optimization problem becomes
\begin{equation}\label{eq:5415problem_dr}
	\begin{array}{rll}
		\min\limits_{\mathbf{z}}     & R_\mathrm{T}(\mathbf{x(\mathbf{z})}) & \mathrm{with} \,\,\, \mathbf{z}\in\mathbb{R}^K\\ 
		& L_{\rm pp}({\mathbf{x(\mathbf{z})}})= L_{\rm pp_0}\\
		& \nabla(\mathbf{x(\mathbf{z})})= \nabla_0,\\
		& |\Delta B(\mathbf{x(\mathbf{z})})|\leq 0.05B_0,\\
		& |\Delta T(\mathbf{x(\mathbf{z})})|\leq 0.05T_0, \\
		& V(\mathbf{\mathbf{x(\mathbf{z})}}) \geq V_0,\\
		&d^{2}_{M}(\mathbf{x}(\mathbf{z}))  \leq \varphi_{\max} \\
		& z_k^{\text{lb}}\leq z_k \leq z_k^{\text{ub}} &  k=1,\dots,K \\
	\end{array}
\end{equation}
Where the bounds for the latent variable are computed by taking the maximum and the minimum of each column component $z_{k}$ of $\mathbf{Z}$ and the value of $\varphi_{\max}$ given in Tab. \ref{tab:dim_mah}. 

The hidden constraints are treated as follows: in case one or more constraints are not satisfied, the geometry will not enter the simulator to obtain the function evaluation, instead, the following pseudo value for the objective function is returned to the optimizer as
\begin{equation}
	f(\mathbf{x}) = h + \psi * \sum_{j=1}^{J}\max\{0, g_j(\mathbf{x})-a_j\}
\end{equation}
where $h=50$, $\psi = 1000$ and $g_j$ and $a_j$ represent the $j$th generic geometrical constraint and its constant term.
\subsection{Hydrodynamic Solver}
The calm-water total resistance is evaluated using the linear potential flow code WARP (Wave Resistance Program), developed at CNR-INM. Wave resistance computations are based on the Dawson (double-model) linearization \cite{dawson1977-NSH}. 
The frictional resistance is estimated using a flat-plate approximation, based on the local Reynolds number \cite{schlichting2000-BLT}. The ship balance (sinkage and trim) is fixed.
Details of equations, numerical implementations, and validation of the numerical solver are given in \cite{bassanini1994-SMI}. 
The convergence grid study can be found in \cite{serani2016-AOR}.

Simulations are performed for the right demi-hull, taking advantage of symmetry about the $\mathsf{x}\mathsf{z}$-plane. The computational domain for the free-surface is defined within $1L_\mathrm{{pp}}$ upstream, $3L_\mathrm{{pp}}$ downstream, and $1.5L_\mathrm{{pp}}$ sideways, for a total of $150\times 44$ grid nodes. The associated hull grid is formed by $180\times 40$ nodes.
\subsection{Numerical Results}
The global optimization process is performed using the DIRECT algorithm \cite{jones1993-JOTA} and Bayesian global optimization \cite{mockus1978application} based on Gaussian processes (GP) \cite{williams2006gaussian}. 
The stopping criteria is the maximum number of function evaluations (i.e. simulation runs) fixed at
\begin{equation}
    I_{\text{max}}=500
\end{equation}

For the GP, we used the lower confidence bound (GP-LCB) \cite{cox1992statistical} as the acquisition function, with the parameter that balances the exploitation/exploration fixed to $\kappa = 1$. 
The LCB is optimized with a multi-start quasi-Newton method BFGS \cite{broyden1970convergence}. 
In this work, a Matérn kernel function \cite{genton2001classes} with a parameter $\nu = 3/2$ has been used for the GP, while its length scale is optimized during the MLE procedure. 
We used a central composite design \cite{box1957multi} without factorial points composed by $2D +1$ points for the initialization of the Bayesian optimization procedure.

In Tab. \ref{tab:results} we reported the numerical results for each optimization algorithm and the three design spaces produced by PPCA ($K=12$), FA ($K=16$), and the full dimensional design space FDS ($M=21$) in terms of the total resistance $R_T(\text{N})$ and percentage reduction of the $R_T[\text{N}]$ with respect to the original geometry. 
The overall best function value obtained is given by the DIRECT algorithm iterating in the subspace produced by the PPCA. 
\begin{table}[htb!]
	\begin{center}
		\caption{Numerical results at the end of the SBDO process.}
		\label{tab:results}
		\small
		\begin{tabular}{lccc}
			\toprule
			Method         & $K$ &  DIRECT $R_T[\text{N}]$ & GP-LCB $R_T[\text{N}]$ \\
			\midrule
			PPCA            & 12   & 34.04 (-18.5\%)  & 37.63 (-9.9\%) \\
			FA 			    & 16   & 35.35 (-15.4\%)  & 38.03 (-9.0\%) \\
			FDS 			& 21   & 40.10 (-4.0\%) & 40.56 (-2.9\%) \\
			\bottomrule
		\end{tabular}
	\end{center}
\end{table}

At the end of the SBDO process, the convergence speed to the achieved optimal function value is evaluated in Fig. \ref{fig:opt_conv}, showing the best objective function value obtained across the iterations from the global optimization algorithms. 
It is evident from these figures that optimizing in the latent space shows a faster convergence rate compared to the FDS. This is true for both algorithms, DIRECT and GP-LCB, as shown in Fig. \ref{fig:conv_direct} and in Fig. \ref{fig:conv_gplcb}, respectively.
\begin{figure}[htb!]
	\centering
	\subfigure[DIRECT]{\label{fig:conv_direct}\includegraphics[width=6cm]{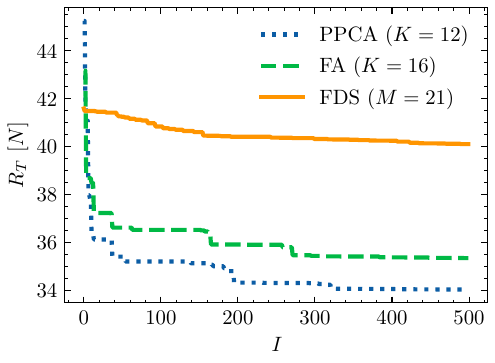}}
	\subfigure[GP-LCB]{\label{fig:conv_gplcb}\includegraphics[width=6cm]{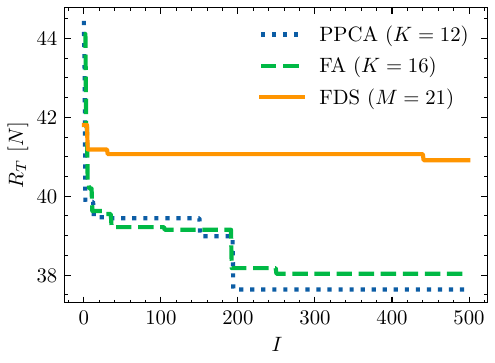}}
	\caption{Convergence of DIRECT and GP-LCB during the optimization in FA, PPCA, and FDS spaces. }
	\label{fig:opt_conv}
\end{figure}

In Figures, \ref{fig:fa_ppca_opt_x} and \ref{fig:fa_ppca_front_opt_x}, as well as Fig. \ref{fig:fa_ppca_rear_opt_x}, the quality of the optimal geometries produced through FA and PPCA can be evaluated from different viewpoints, along with their Mahalanobis distances. 
Overall, is possible to notice a good geometrical regularity across all the optimal solutions. 
Only the optimal solutions found by DIRECT reveal a bulb region slightly sharper (Fig. \ref{fig:front_opt_b} and \ref{fig:front_opt_d}) with a marginally more noticeable curvature over the hull only for the FA case (Fig. \ref{fig:rear_opt_d}). 
Let's observe that those geometries obtained a value of the Mahalanobis distance respectively of $d^2_{M}(\mathbf{x}^\star) =21.73$ and $d^2_{M}(\mathbf{x}^\star) =27.12$ where their maximum value allowed given by $\varphi_{\max} = 21.74$ and  $\varphi_{\max} = 27.13$ for PPCA and FA respectively (Tab. \ref{tab:dim_mah}). 
This outcome indicates that DIRECT tried to generate hull designs with a more pointed bulb. However, the Mahalanobis distance constraint in Eq. \ref{eq:rpopt_sbdo3} limits significant and abnormal alterations to that specific section of the hull.
%
\begin{figure}[htb!]
	\centering
	\subfigure[$d^2_{M}(\bar{\mathbf{x}}) = 0$]{\label{fig:opt_a}\includegraphics[width=0.23\textwidth]{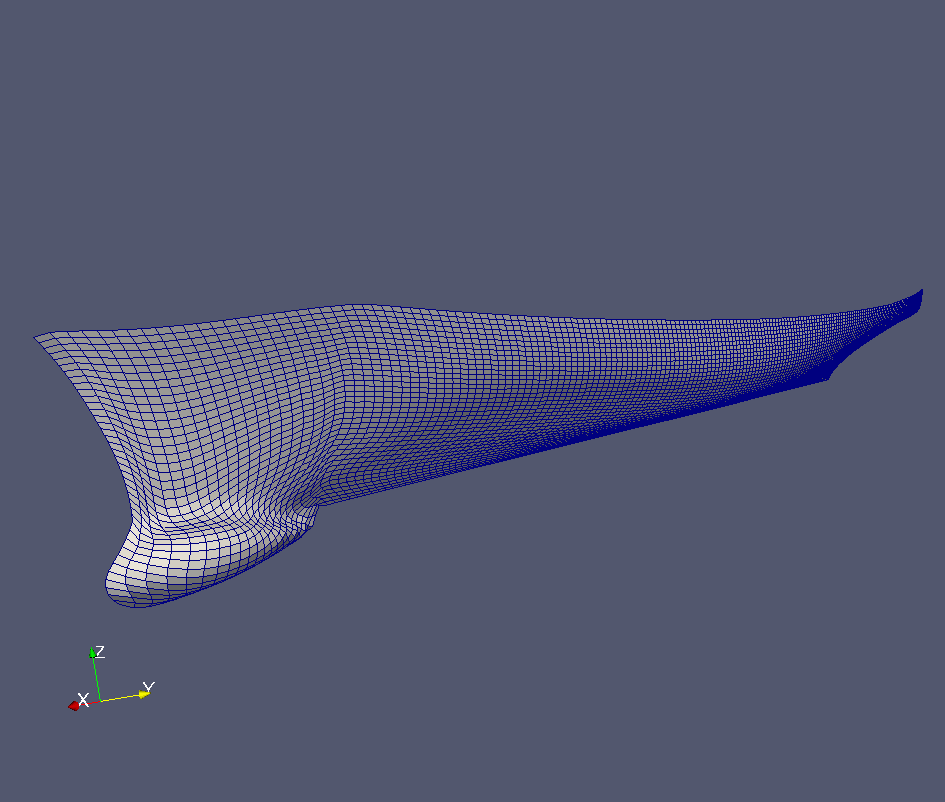}}\\
	\subfigure[$d^2_{M}(\mathbf{x}^\star) = 21.73$]{\label{fig:opt_b}\includegraphics[width=0.23\textwidth]{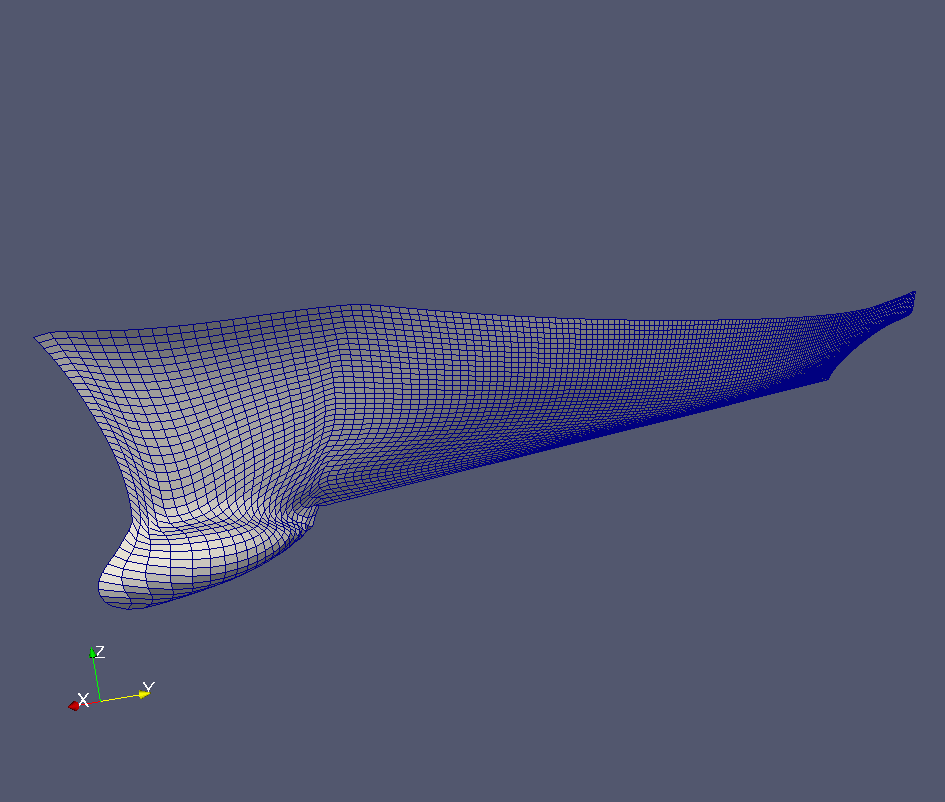}} 
	\subfigure[$d^2_{M}(\mathbf{x}^\star) = 6.65$]{\label{fig:opt_c}\includegraphics[width=0.23\textwidth]{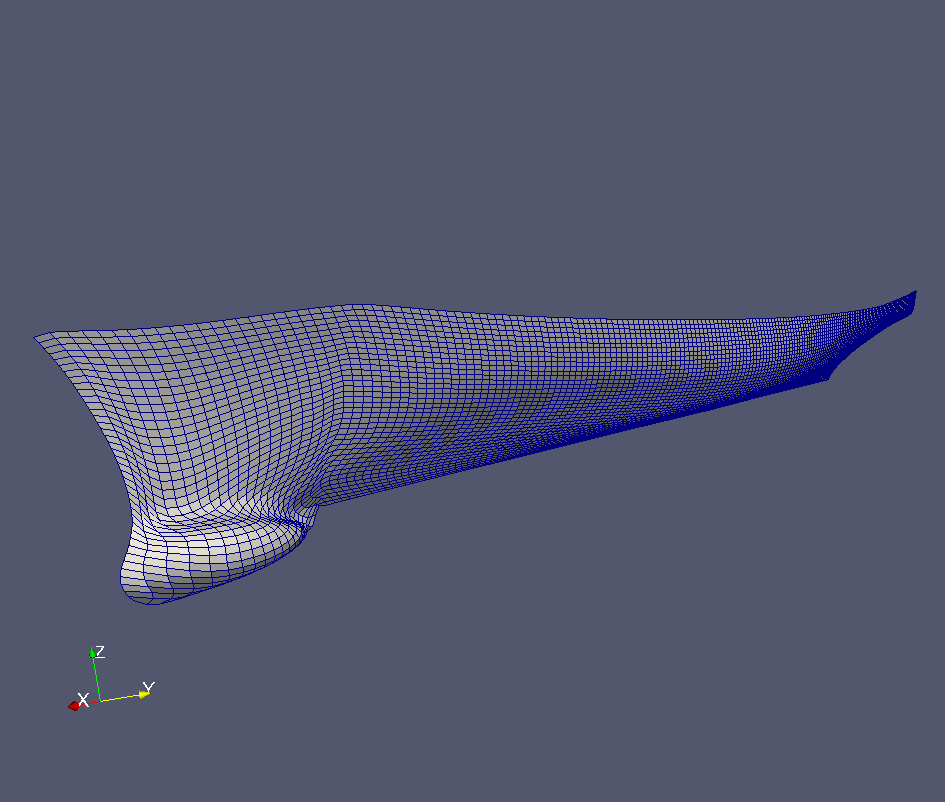}} 
	\subfigure[$d^2_{M}(\mathbf{x}^\star) =27.12$]{\label{fig:opt_d}\includegraphics[width=0.23\textwidth]{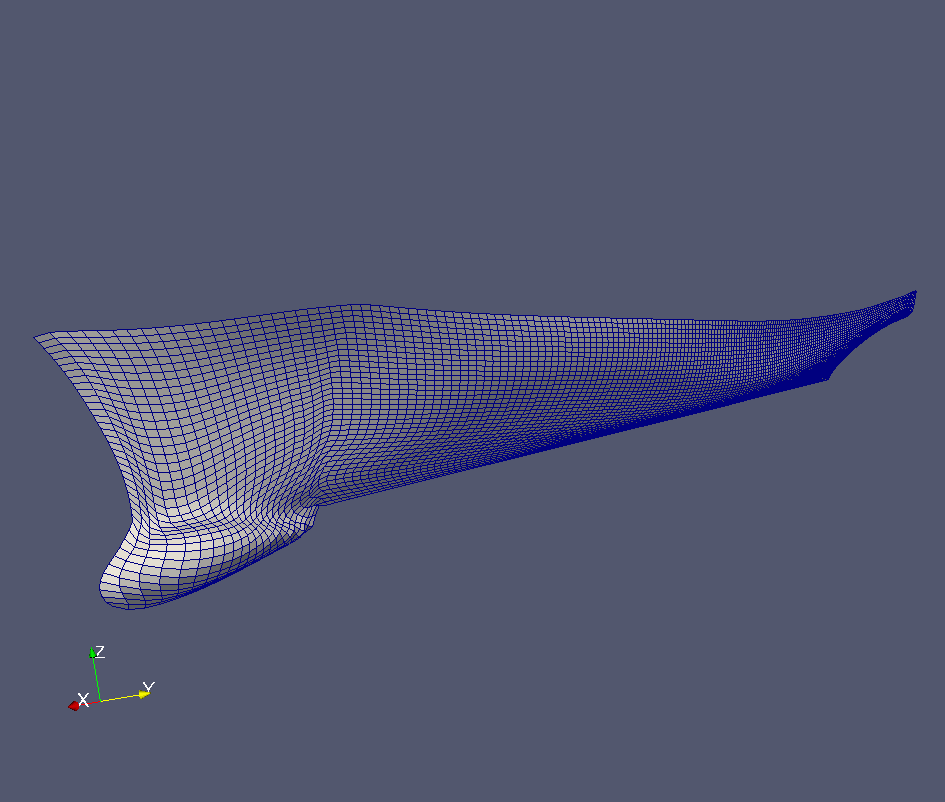}} 
	\subfigure[$d^2_{M}(\mathbf{x}^\star) =16.17$]{\label{fig:opt_e}\includegraphics[width=0.23\textwidth]{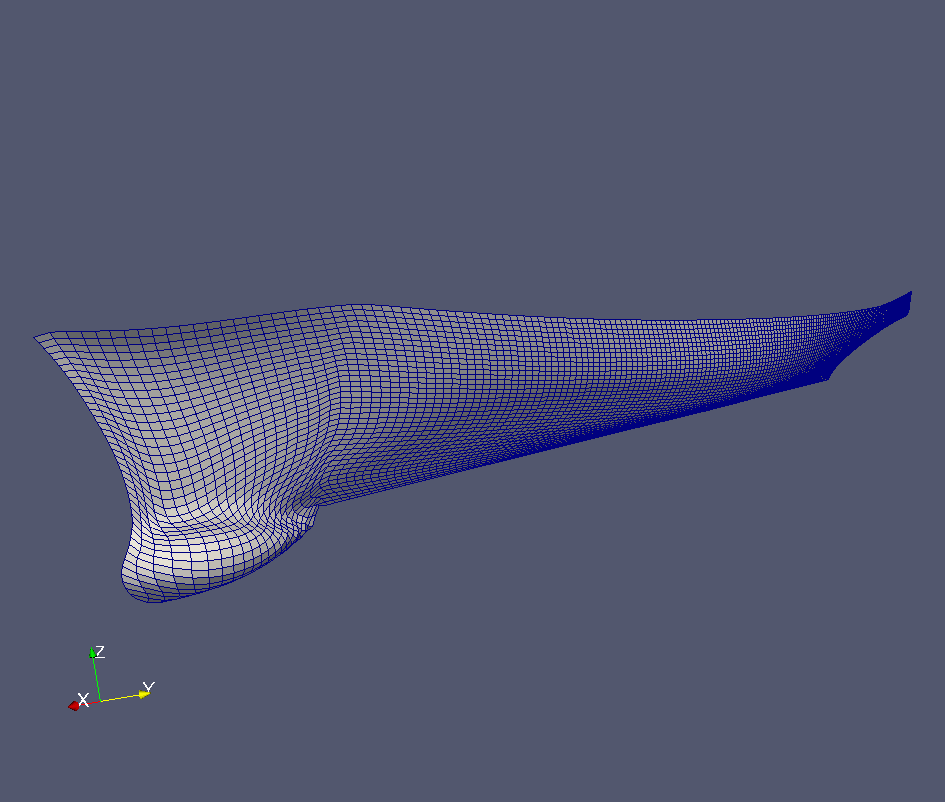}}
	\caption{Optimal solutions $\mathbf{x}^\star$ and relative Mahalanobis distance for PPCA/DIRECT (b), PPCA/GP-LCB (c), FA/DIRECT (d) and FA/GP-LCB (e). The mean geometry $\bar{\mathbf{x}}$ in (a).}
	\label{fig:fa_ppca_opt_x}
\end{figure}
\begin{figure}[htb!]
	\centering
	\subfigure[$d^2_{M}(\bar{\mathbf{x}}) = 0$]{\label{fig:front_opt_a}\includegraphics[width=0.23\textwidth]{figs/MEAN_FRONT_BULB.png}}\\
	\subfigure[$d^2_{M}(\mathbf{x}^\star) = 21.73$]{\label{fig:front_opt_b}\includegraphics[width=0.23\textwidth]{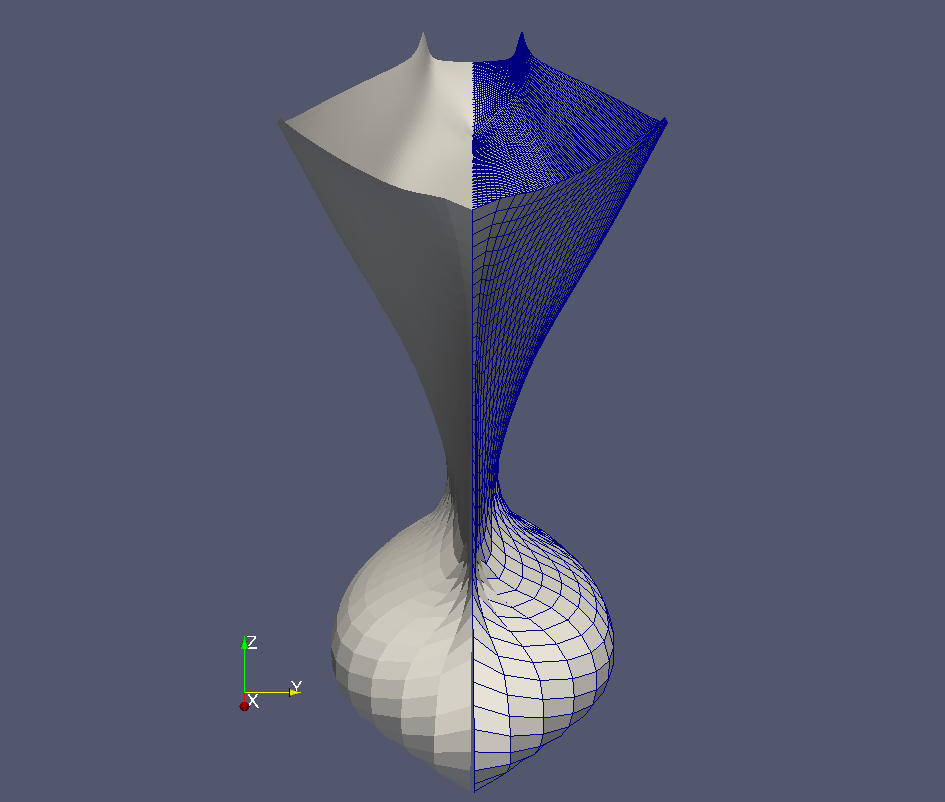}} 
	\subfigure[$d^2_{M}(\mathbf{x}^\star) = 6.65$]{\label{fig:front_opt_c}\includegraphics[width=0.23\textwidth]{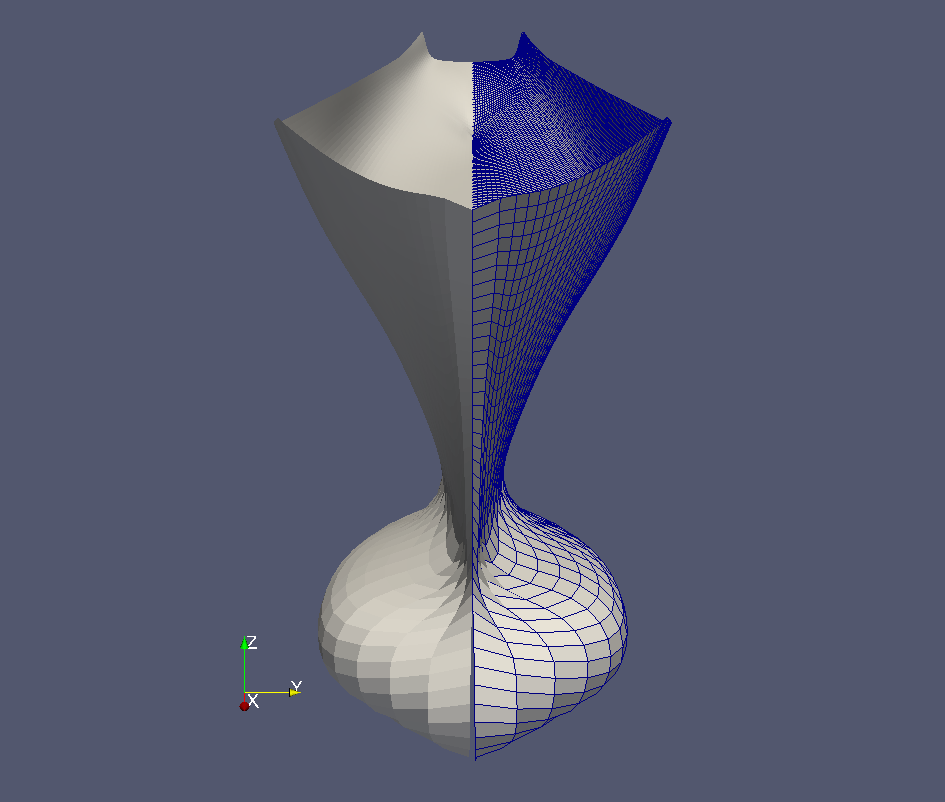}} 
	\subfigure[$d^2_{M}(\mathbf{x}^\star) =27.12$]{\label{fig:front_opt_d}\includegraphics[width=0.23\textwidth]{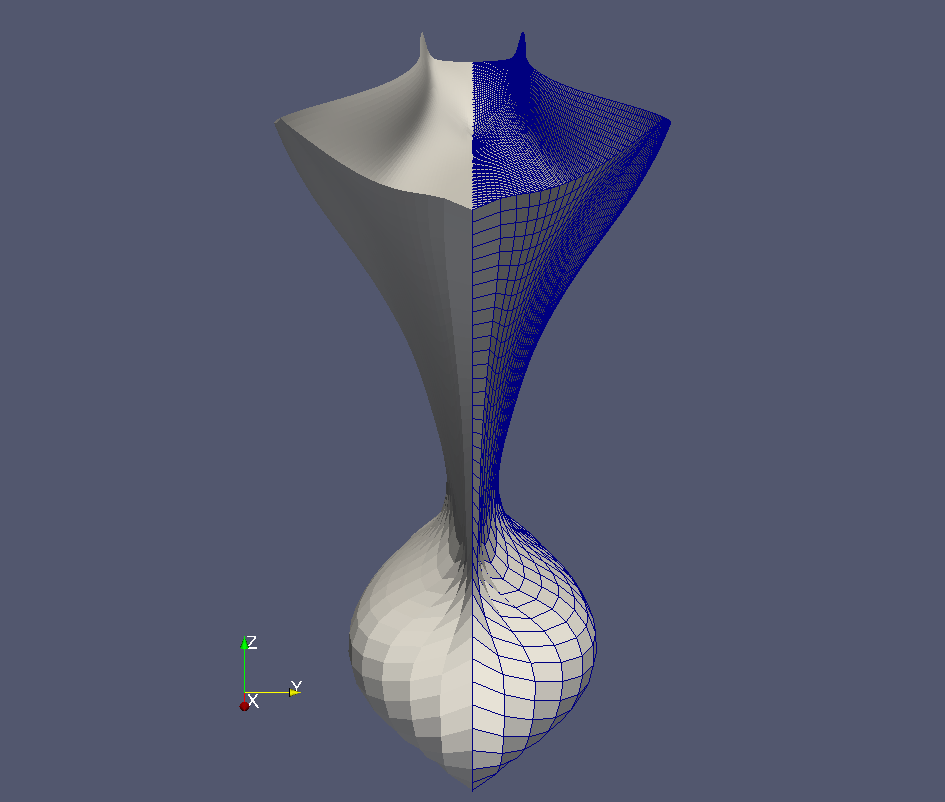}} 
	\subfigure[$d^2_{M}(\mathbf{x}^\star) =16.17$]{\label{fig:front_opt_e}\includegraphics[width=0.23\textwidth]{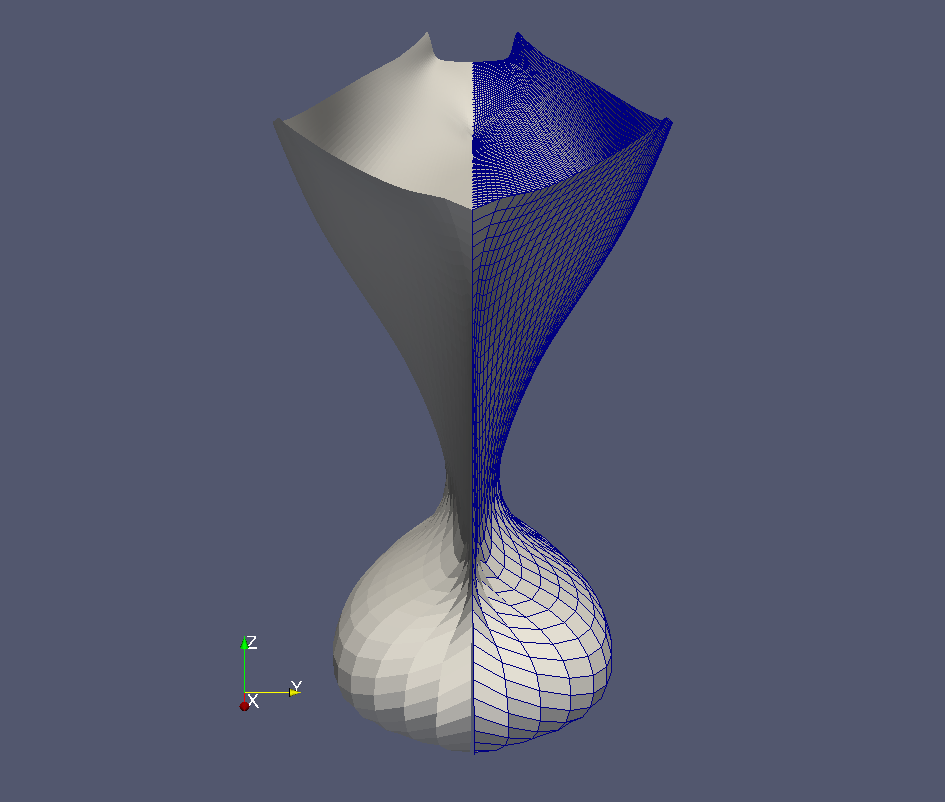}}
	\caption{Optimal solutions (bow region) $\mathbf{x}^\star$ and relative  Mahalanobis distance for PPCA/DIRECT (b), PPCA/GP-LCB (c), FA/DIRECT (d) and FA/GP-LCB (e). The mean geometry $\bar{\mathbf{x}}$ in (a).}
	\label{fig:fa_ppca_front_opt_x}
\end{figure}
\begin{figure}[htb!]
	\centering
	\subfigure[$d^2_{M}(\bar{\mathbf{x}}) = 0$]{\label{fig:rear_opt_a}\includegraphics[width=0.23\textwidth]{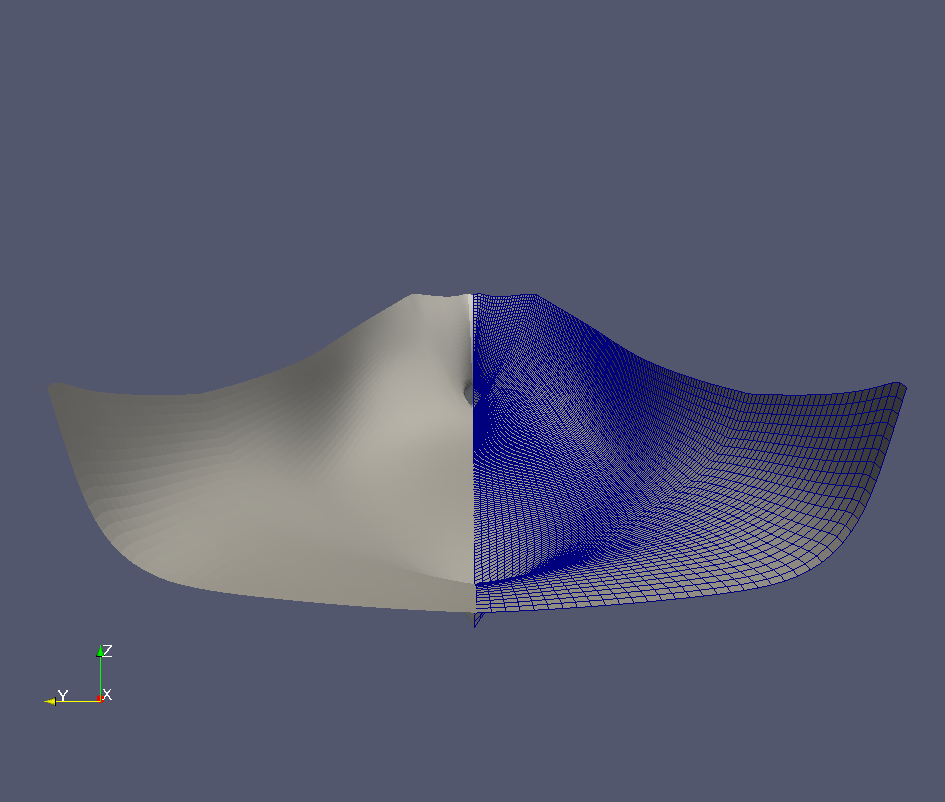}}\\
	\subfigure[$d^2_{M}(\mathbf{x}^\star) = 21.73$]{\label{fig:rear_opt_b}\includegraphics[width=0.23\textwidth]{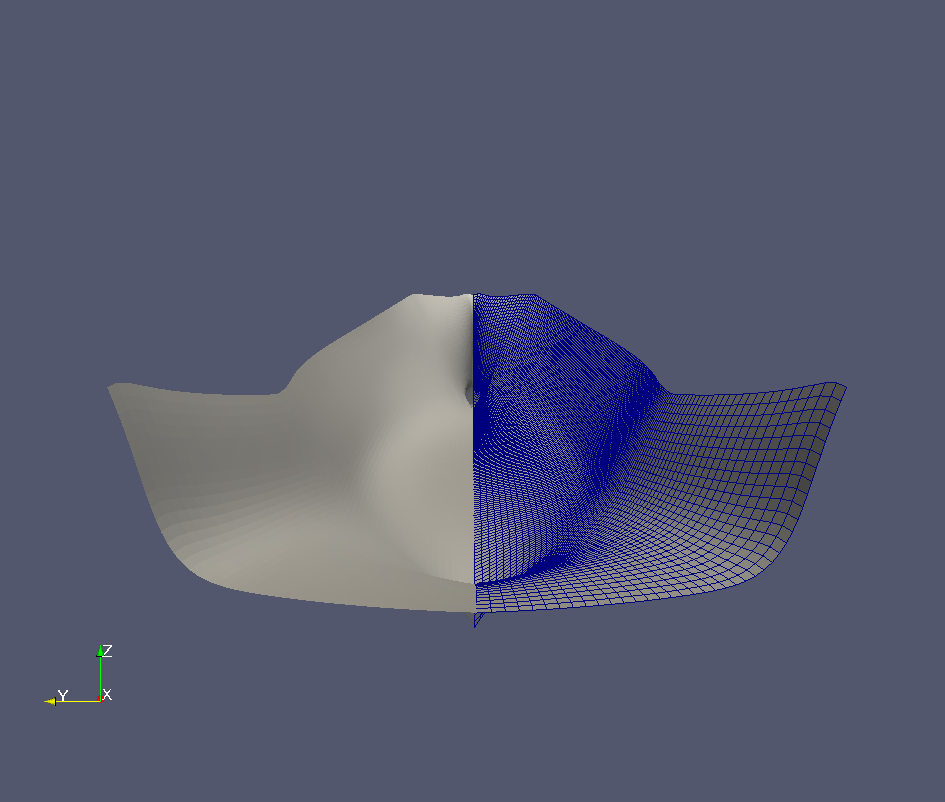}} 
	\subfigure[$d^2_{M}(\mathbf{x}^\star) = 6.65$]{\label{fig:rear_opt_c}\includegraphics[width=0.23\textwidth]{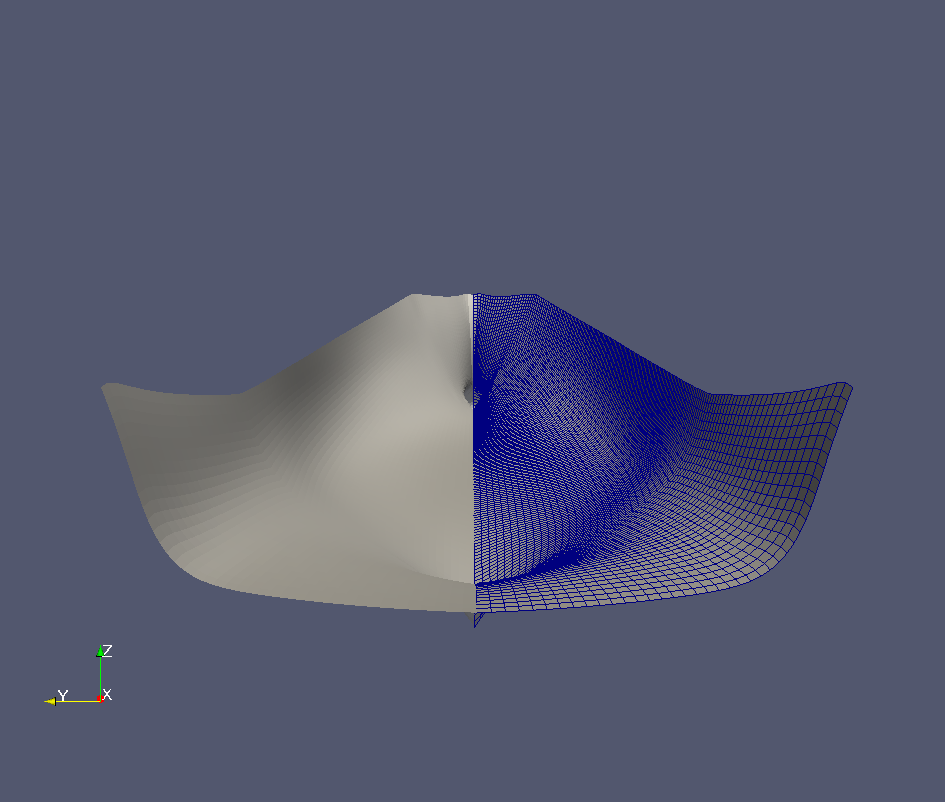}} 
	\subfigure[$d^2_{M}(\mathbf{x}^\star) =27.12$]{\label{fig:rear_opt_d}\includegraphics[width=0.23\textwidth]{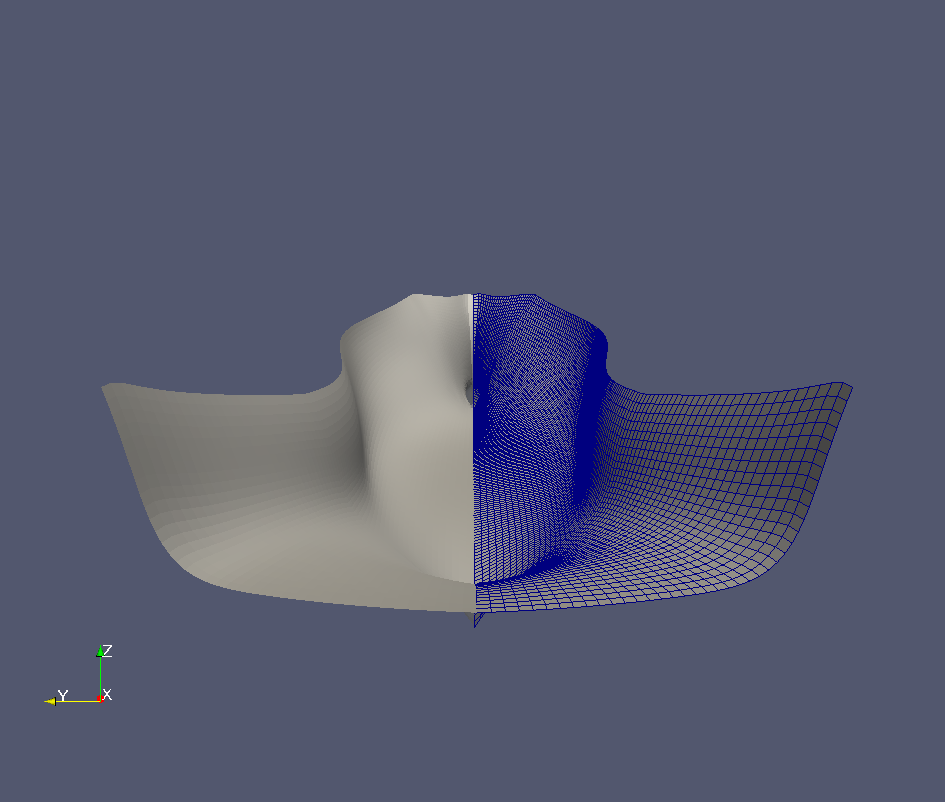}} 
	\subfigure[$d^2_{M}(\mathbf{x}^\star) =16.17$]{\label{fig:rear_opt_e}\includegraphics[width=0.23\textwidth]{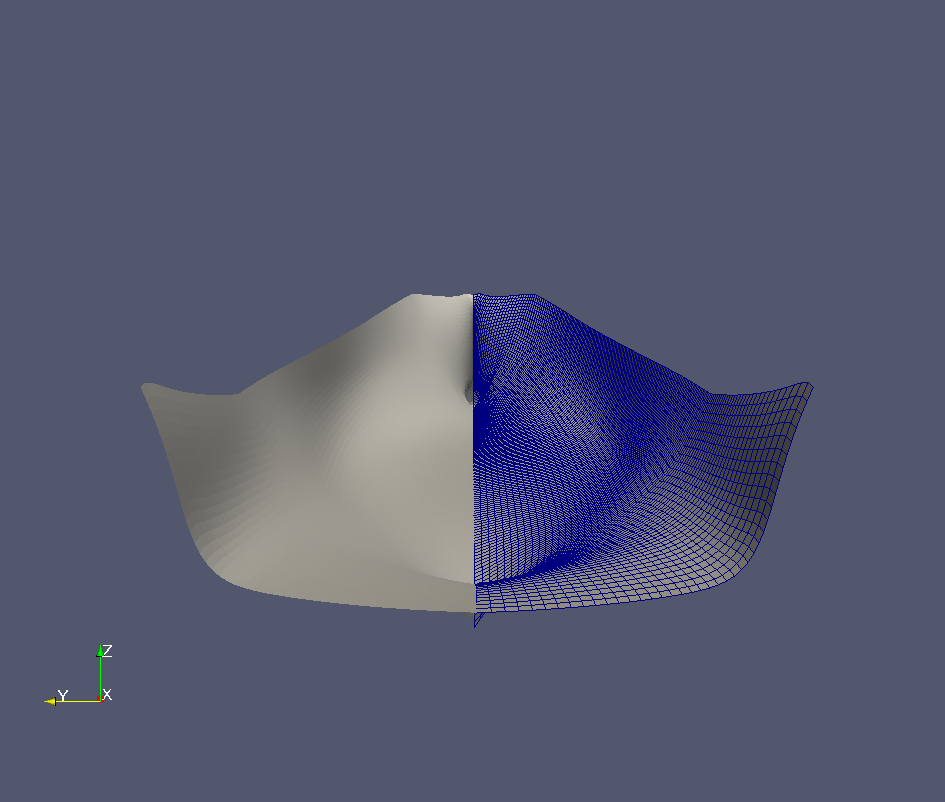}}
	\caption{Optimal solutions (stern region) $\mathbf{x}^\star$ and relative Mahalanobis distance for PPCA/DIRECT (b), PPCA/GP-LCB (c), FA/DIRECT (d) and FA/GP-LCB (e). The mean geometry $\bar{\mathbf{x}}$ in (a).}
	\label{fig:fa_ppca_rear_opt_x}
\end{figure}

We applied the SBDO process also for the PCA with $K=12$. In Fig. \ref{fig:pca_conv}, we can observe the convergence of the two optimizers used in this application using the PCA subspace.
\begin{figure}[htb!]
	\centering
	\subfigure[DIRECT]{\label{fig:conv_pca_direct}\includegraphics[width=6cm]{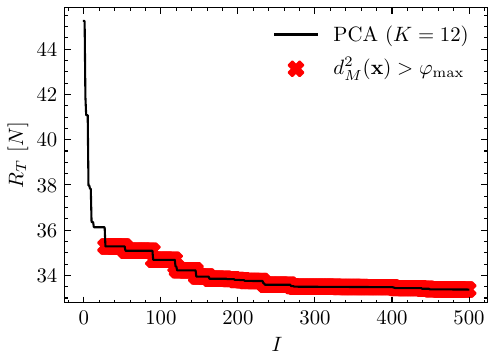}}
	\subfigure[GP-LCB]{\label{fig:conv_pca_gplcb}\includegraphics[width=6cm]{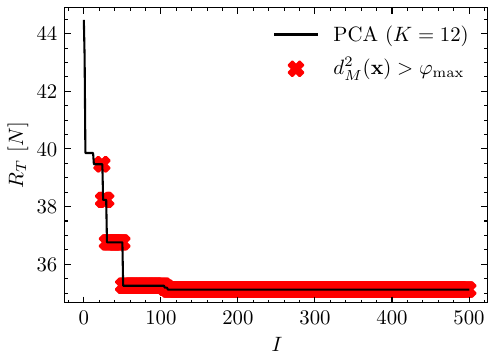}}
	\caption{Convergence of DIRECT and GP-LCB during the optimization in PCA space. In red a geometry that does not satisfy constraint based on the Mahalanobis distance.}
	\label{fig:pca_conv}
\end{figure}
To evaluate if the geometries generated by the SBDO process using the PCA subspace are within the distribution of the training data, we calculated the squared Mahalanobis distance using the inverse of the covariance matrix $\mathbf{C}^{-1}$ estimated by the PPCA model.
\begin{figure}[htb!]
	\centering
	\subfigure[DIRECT]{\label{fig:mah_pca_pdf_direct}\includegraphics[width=12cm]{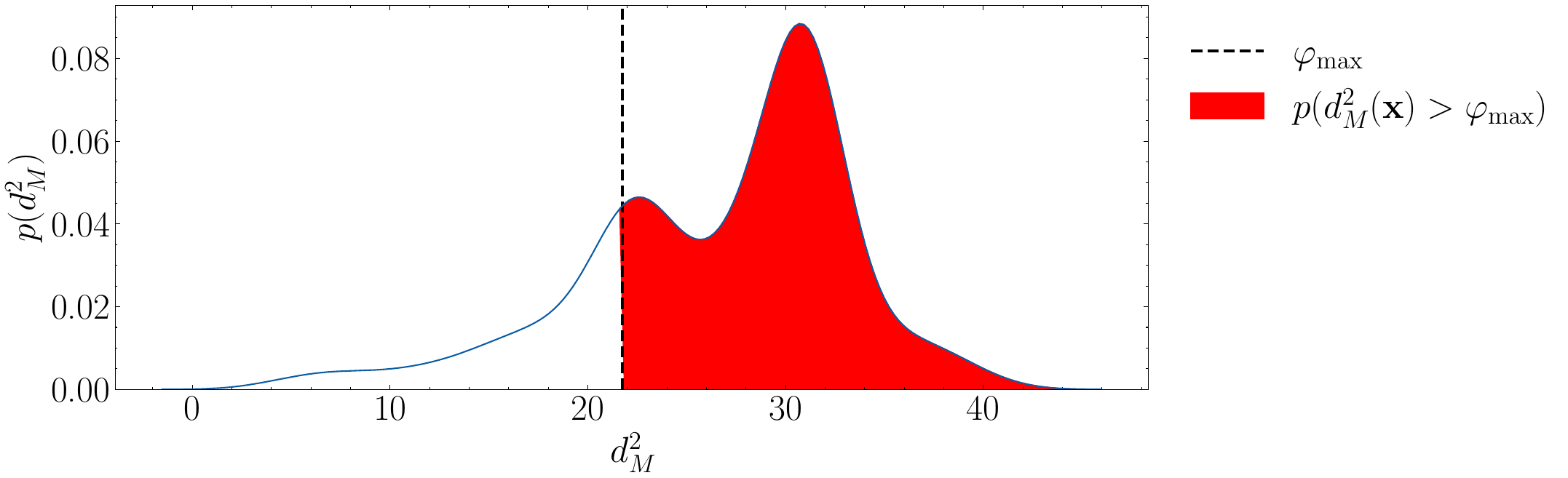}}
	\subfigure[GP-LCB]{\label{fig:mah_pca_pdf_gplcb}\includegraphics[width=12cm]{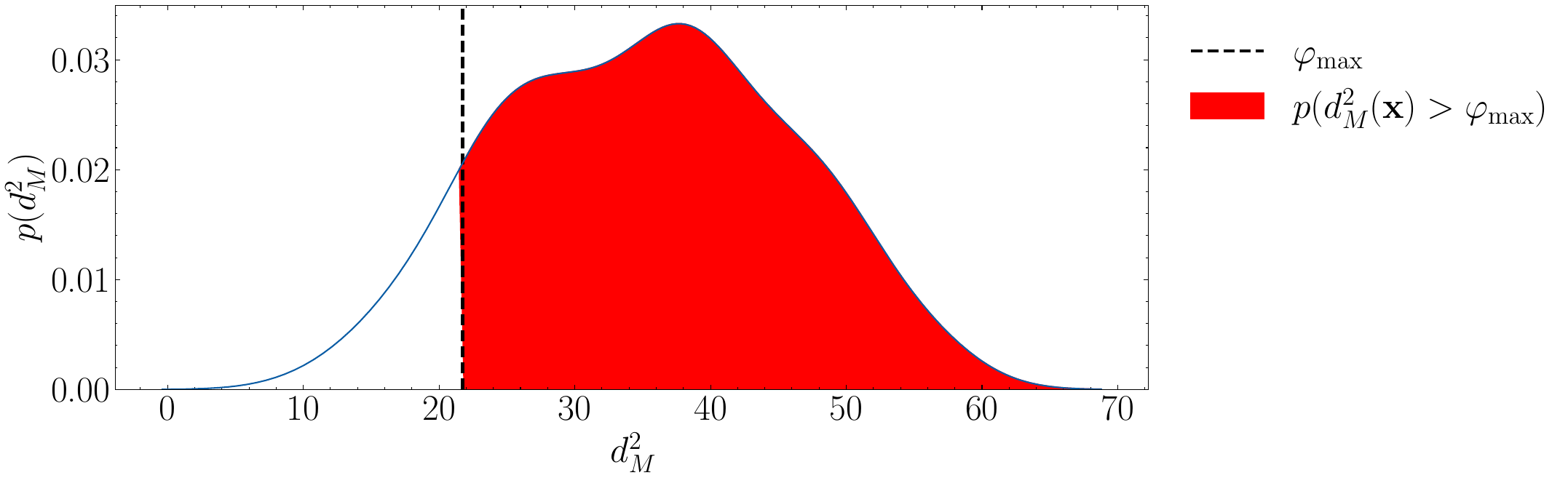}}
	\caption{Empirical PDF for the squared Mahalanobis distance for the geometries projected back from the PCA subspace for DIRECT and GP-LCB.}
\label{fig:mah_pca_pdfopt}
\end{figure}
Our findings suggest that most of the designs generated from PCA fail to satisfy the PPCA Mahalanobis distance constraint in Eq. \ref{eq:rpopt_sbdo3}. 
This is evidenced in Fig. \ref{fig:conv_pca_direct} and in Fig. \ref{fig:conv_pca_gplcb}, where the red cross highlights the geometries with the best current objective function value that do not satisfy the constraint based on the Mahalanobis distance.
Fig. \ref{fig:mah_pca_pdfopt} displays the empirical PDF function for the squared Mahalanobis distance of the geometries generated using DIRECT and GP-LCB optimizers. 
For GP-LCB, approximately $88.4\%$ of geometrically feasible geometries produced by PCA do not satisfy the PPCA Mahalanobis distance constraint, whereas, for DIRECT, this fraction is approximately $81.6\%$ of the total number of function evaluations performed. 
This result may suggest that most simulation runs are wasted on potentially anomalous geometries. 
To investigate this possibility, in Fig. \ref{fig:comp_direct_ano} and in Fig. \ref{fig:comp_bgo_ano} there is a comparison between a design (geometrically feasible) sampled from DIRECT, GP-LCB, and the mean geometry respectively. 
In our analysis, we discovered that the two optimizers followed distinct trajectories within the latent space. 
The DIRECT algorithm primarily targeted modifications to the bulb region to find an optimal solution, whereas GP-LCB placed greater emphasis on altering the stern region of the hull. 
Hence, many designs generated by DIRECT show a behavior similar to the geometry shown in Fig. \ref{fig:pca_direct_bulb_ano}, where most of the bulb volume is concentrated in its rear part producing an exaggeratedly sharp configuration, similar to the one already discussed in Fig. \ref{fig:ano_3}. In parallel, the GP-LCB sampled many designs close to the one shown in Fig. \ref{fig:pca_bgo_rear_ano} where an abnormally hard inflection with an acute angle dominates the stern part of the hull.

\begin{figure}[htb!]
	\centering
	\subfigure[$d^2_{M}(\mathbf{x}) = 39.66$]{\label{fig:pca_direct_bulb_ano}\includegraphics[width=0.47\textwidth]{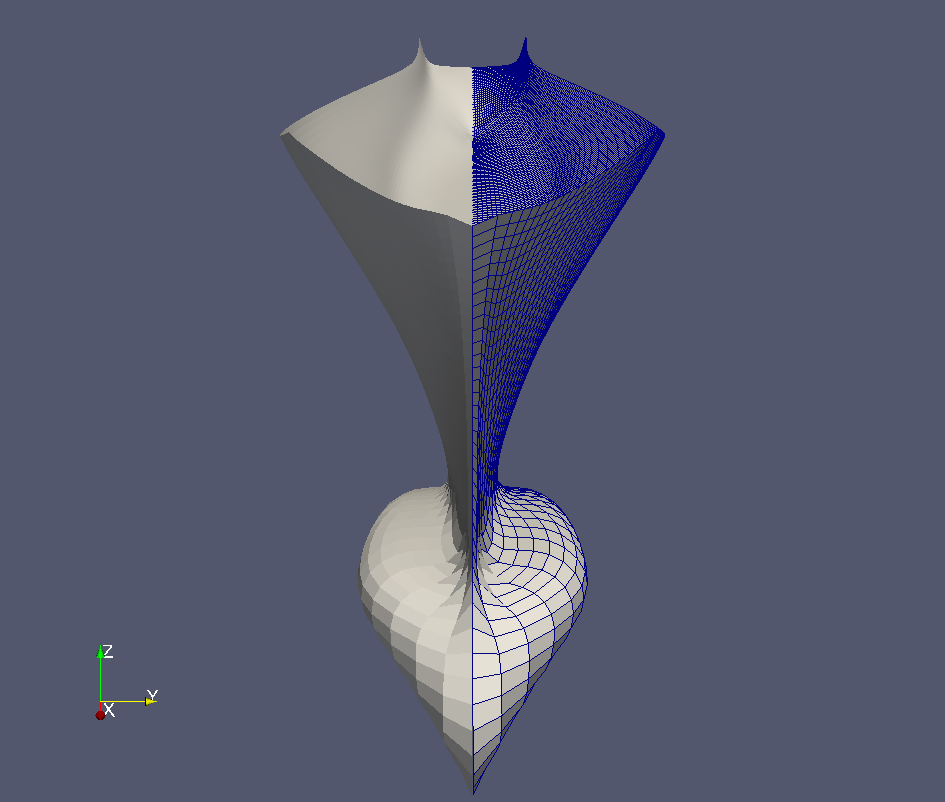}}
	\subfigure[$d^2_{M}(\bar{\mathbf{x}}) = 0$]{\label{fig:mean_bulb}\includegraphics[width=0.47\textwidth]{figs/MEAN_FRONT_BULB.png}} 
	\caption{Example of an anomalous geometry but geometrically feasible from DIRECT using PCA subspace in (a). The mean geometry $\bar{\mathbf{x}}$ in (b).}
\label{fig:comp_direct_ano}
\end{figure}
\begin{figure}[htb!]
	\centering
	\subfigure[$d^2_{M}(\mathbf{x}) = 49.72$]{\label{fig:pca_bgo_rear_ano}\includegraphics[width=0.49\textwidth]{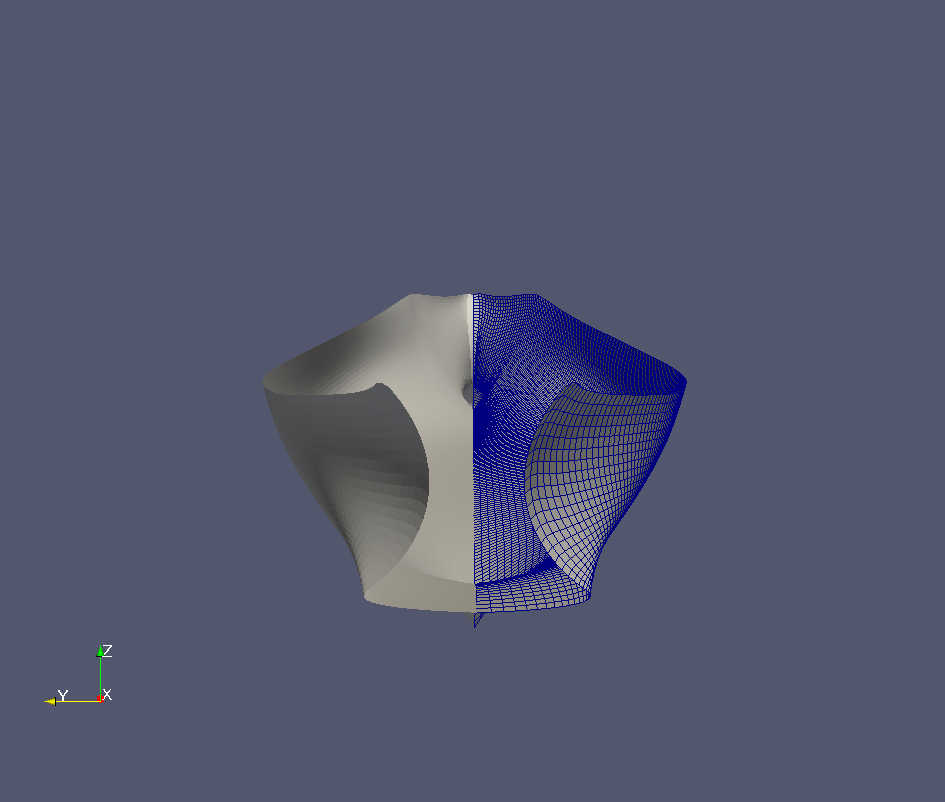}}
	\subfigure[$d^2_{M}(\bar{\mathbf{x}}) = 0$]{\label{fig:mean_rear}\includegraphics[width=0.49\textwidth]{figs/MEAN_REAR.png}} 
	\caption{Example of an anomalous geometry but geometrically feasible from the GP-LCB using PCA subspace in (a). The mean geometry $\bar{\mathbf{x}}$ in (b).}
\label{fig:comp_bgo_ano}
\end{figure}

\section{Conclusion and Future Work}\label{sec:8}
In this paper, we proposed a new framework for SBDO in shape optimization which focuses on the definition of a lower dimensional space to speed up the convergence of global optimization algorithms while promoting only high-quality designs during the process. 
For this purpose, we used and compared two probabilistic linear latent variable models such as PPCA and FA. 
Those methods, identified a lower dimensional space of $K=12$ and $K=16$ for PPCA and FA respectively with a reduction of about $42\%$ and $23\%$ respect to the original dimensionality ($M=21$) while maintaining the $99\%$ of the geometrical variance. 

Through PPCA and FA, we also performed density estimation. 
We showed that when the shape modification method is linear and the design variables are sampled uniformly at random, the generated data follows approximately a Gaussian distribution. This is a consequence of a direct application of the central limit theorem (CLT). For this reason, PPCA and FA are demonstrated to be ideal for the current application since they both assume that the underlying generative process of the data is modeled by a Gaussian distribution.
The degree of anomalousness is measured in terms of Mahalanobis distance, and the paper demonstrates that abnormal designs tend to exhibit a high value of this metric. 
This enables the definition of a new optimization model where anomalous geometries are penalized and consequently avoided during the optimization loop.

The new proposed framework is demonstrated for the shape optimization of naval military US destroyer. 
The optimization results carried out with two global optimization algorithms DIRECT and Bayesian optimization (GP-LCB), showed that reducing the dimensionality allows a greater reduction in the objective function values and a faster convergence with respect to solving the original problem in the full-dimensionality space. Furthermore, we observed high-quality designs with the absence of anomalous curvatures and deflections across all the geometries generated using our new framework.
We also performed the same optimization process using simple PCA, showing that a large number of simulations are wasted for geometries with abnormal geometric characteristics.

Our future work will concentrate on using our framework in different real-world applications, such as airfoil design. We'll also compare its performance with other dimensionality reduction models, both unsupervised \cite{chen2020airfoil, serani2023parametric} and supervised \cite{d2023learning, lukaczyk2014-AIAA}.
\section*{Acknowledgement}
The work of Danny D'Agostino was partially developed in \cite{d2021lipschitzian} while performing his Ph.D. at the Sapienza University of Rome and as a member of the multidisciplinary optimization group at the Institute of Marine Engineering of the Italian National Research Council CNR-INM.

\bibliographystyle{plain}
\bibliography{biblio}  %

\end{document}